\documentclass{article}
\usepackage[preprint]{neurips_2026}
\usepackage{microtype}
\usepackage{graphicx}
\usepackage{subcaption}
\usepackage{booktabs} 
\usepackage{comment}
\usepackage{hyperref}
\usepackage{url}

\usepackage{graphicx}
\usepackage{wrapfig}
\usepackage{float}
\usepackage{multirow} 
\usepackage{enumitem}
\usepackage[most]{tcolorbox}
\usepackage{amsthm}
\usepackage{amsmath}
\usepackage[most]{tcolorbox}  
\usepackage{amssymb}
\usepackage{caption}
\usepackage{subcaption}
\usepackage{url}            
\usepackage{booktabs}       
\usepackage{amsfonts}    
\usepackage{csquotes}
\usepackage{nicefrac}       
\usepackage{microtype}      
\usepackage{xcolor}         
\usepackage{subcaption}
\usepackage{amsmath} 

\usepackage{algorithm}
\usepackage{algpseudocode}
\usepackage{algorithmicx}
\usepackage[table]{xcolor}
\usepackage{titlesec}
\titlespacing*{\subsection}{0pt}{0.5ex}{0.5ex}
\titlespacing*{\section}{0pt}{0.8ex}{0.5ex}
\titlespacing*{\subsubsection}{0pt}{0.5ex}{0.5ex}
\titlespacing*{\paragraph}{0pt}{0.5ex}{0.5ex}
\usepackage{hyperref}

\usepackage{amsmath}
\usepackage{amssymb}
\usepackage{mathtools}
\usepackage{amsthm}

\usepackage[capitalize,noabbrev]{cleveref}

\theoremstyle{plain}
\newtheorem{theorem}{Theorem}[section]
\newtheorem{proposition}[theorem]{Proposition}

\theoremstyle{definition}

\theoremstyle{remark}

\usepackage{tabularx}
\usepackage{natbib}

\usepackage[textsize=tiny]{todonotes}

\title{Adapting Actively on the Fly: Relevance-Guided Online Meta-Learning with Latent Concepts for Geospatial Discovery}

\author{
Jowaria Khan$^{1}$ \qquad
Anindya Sarkar$^{2}$ \qquad
Yevgeniy Vorobeychik$^{2}$ \qquad
Elizabeth Bondi-Kelly$^{1}$ \\[0.5em]
$^{1}$University of Michigan, Ann Arbor, MI, USA \\
$^{2}$Washington University in St. Louis, St. Louis, MO, USA \\
}

\begin{document}

\maketitle

\begin{abstract}

In environmental monitoring, data collection is often costly, sparse, and shaped by urgent public-health needs. This is particularly true for cancer-causing PFAS (Per- and polyfluoroalkyl substances) contamination, where discussions with domain experts and environmental organizations highlight the need to strategically identify high-risk, under-observed regions under tight sampling budgets. More broadly, similar challenges arise in disaster response and public health settings, where dynamic environments make it essential to efficiently uncover hidden targets from limited ground truth. Yet sparse and biased geospatial labels limit the applicability of existing learning-based methods, such as reinforcement learning. To address this, we propose a unified geospatial discovery framework that integrates active learning, online meta-learning, and concept-guided reasoning. Our approach introduces two key innovations built on a shared notion of \emph{concept relevance}, capturing how domain-specific factors influence target presence: a \emph{concept-weighted uncertainty sampling strategy}, where uncertainty is modulated by learned relevance from readily available concepts such as land cover and source proximity; and a \emph{relevance-aware meta-batch formation strategy} that promotes semantic diversity during online-meta updates, improving generalization in dynamic environments. We evaluate our framework on PFAS contamination discovery as a real-world inspired environmental monitoring task, demonstrating robust target discovery under limited data and changing conditions.

\end{abstract}
\section{Introduction}

Discovering targets of interest in large, costly-to-sample, dynamic spaces is a core challenge across environmental monitoring, disaster response, public health, and other geospatial tasks, where targets may include pollution hotspots, damaged regions, or areas at risk of disease \citep{Bondi_Chen_Golden_Behari_Tambe_2022}. These settings often involve strict acquisition budgets, limited prior observations, and the need to make sampling decisions sequentially over time. This challenge is central to PFAS contamination monitoring, where discussions with environmental domain experts motivate the need to identify high-risk, under-observed regions despite sparse and expensive ground truth. In fact, PFAS and other  phenomena may be diffuse and not directly observable in remote sensing imagery \citep{Ayush_Uzkent_Burke_Lobell_Ermon_2020}. 
Decision-making in such environments must therefore be guided by uncertain and sparse prior data, balancing exploration to gather new information, with exploitation to focus on regions most likely to contain areas of interest.

Prior geospatial search methods~\citep{sarkar2023partially,Sarkar_Lanier_Alfeld_Feng_Garnett_Jacobs_Vorobeychik_2024} often rely on reinforcement learning (RL) or POMDP-based approaches~\citep{kaelbling1998planning,ross2008online}, which typically require large numbers of interactions - even in toy domains like Atari~\citep{lattimore2013sample,zhang2023settling,schulman2017proximal} - or extensive labeled data. Such requirements are impractical in real-world settings where labels are costly and query budgets are limited within each monitoring cycle, e.g., $\sim$100 samples per year for field-based environmental monitoring, where each label requires substantial sampling effort~\citep{US_EPA_2018}.


\begin{wrapfigure}{r}{0.40\columnwidth}
  \centering
  \vspace{-10pt}
  \includegraphics[width=\linewidth]{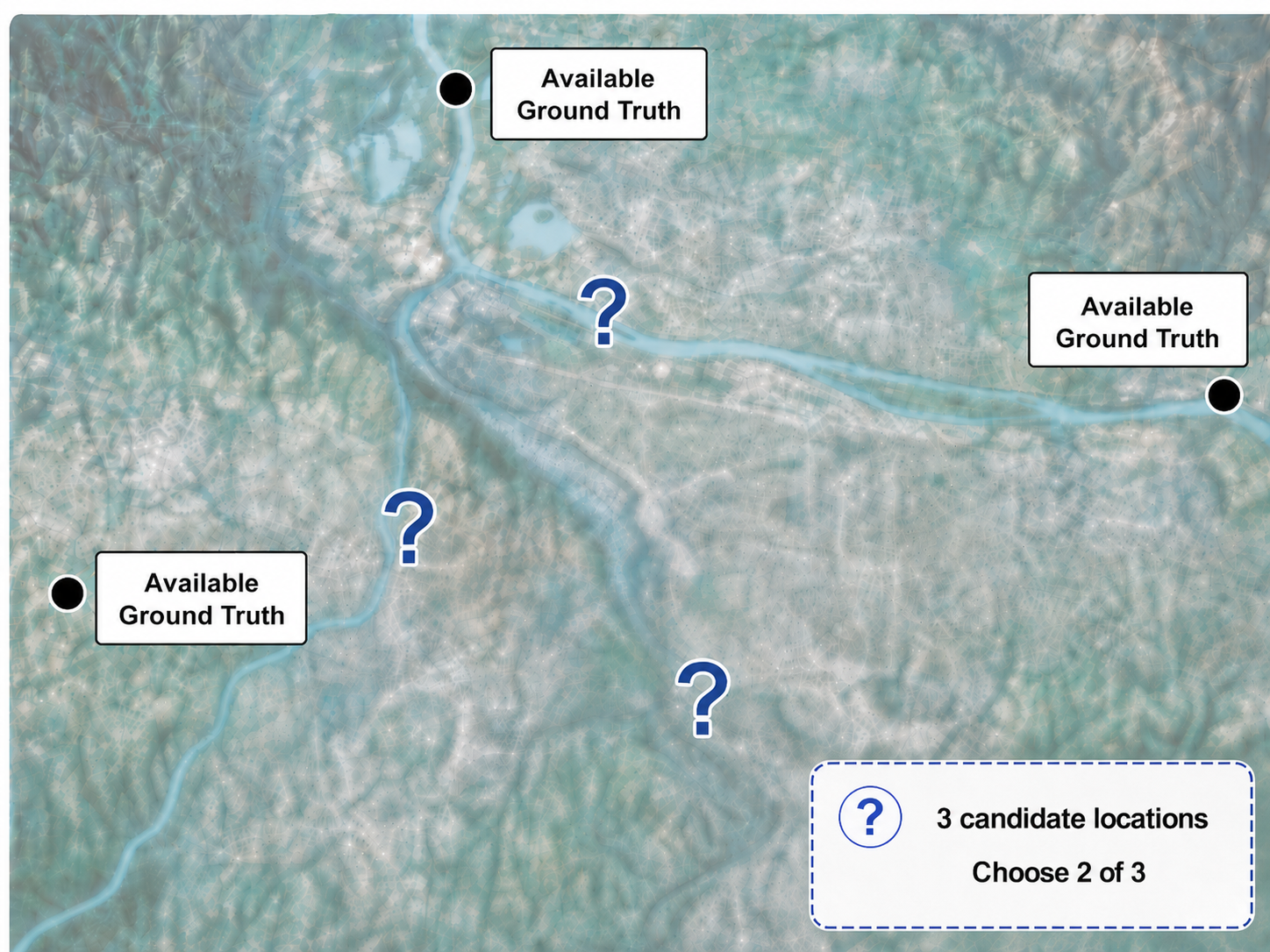}
  \caption{\small{Learning under limited sampling budget: selecting the most informative locations.}}
  \label{fig:yourlabel}
  \vspace{-10pt}
\end{wrapfigure}
Classical bandit algorithms also offer an efficient way to balance exploration and exploitation, but they lack mechanisms to exploit spatial or semantic structure in high-dimensional data~\citep{li2010contextual, bubeck2012regret,ban2024neural}, which is critical 
in  settings like geospatial search, with structured and highly correlated observations. Active and online learning techniques hold relevance, yet the dynamic nature of areas of interest, such as pollution or disease, poses substantial challenges. Related lifelong learning approaches (~\cite{zhu2025adaptive,vogelstein2025simple,mendez2023embodied}) also aim to handle non-stationarity, but typically rely on revisitation, replay, or task boundaries \cite{kirkpatrick2017overcoming,shin2017continual}, which are unavailable in our setting because regions arrive sequentially, labels require costly field sampling, and the model must adapt without repeatedly revisiting all previously queried sites. Conventional meta-learning methods~\citep{finn2017model,finn2018probabilistic,finn2019online}, designed to rapidly adapt policies, cannot be directly applied since we do not have  a predefined meta-training set. 
Furthermore, geospatial search in open-world settings, where regions are encountered sequentially without a fixed or revisitable data pool, imposes severe constraints. For example, the environment itself evolves over time, requiring the model’s memory to evolve accordingly rather than retain all past observations.
Therefore, identifying high-utility samples requires novel adaptive strategies. We formalize these challenges through our newly proposed \emph{Open-World Learning for Geospatial Prediction and Sampling (OWL-GPS)} problem setting.

\vspace{-5pt}
\begin{tcolorbox}[colback=gray!10, colframe=gray!40!black, title=OWL-GPS Assumptions]
\textbf{(OWL-GPS)} operates under three key constraints: 
\textbf{(i)} inputs arrive sequentially and must be actively selected and acted upon by the learned policy under a non-stationary distribution, 
\textbf{(ii)} once observed, samples have a limited \emph{lifespan} (see Sec.~\ref{sec:method}): they cannot be repeatedly revisited under limited budgets, or replayed indefinitely as the environment evolves, and 
\textbf{(iii)} sampling budgets are strictly limited both during training and inference. 
\end{tcolorbox}
These constraints reflect practical realities in environmental monitoring pipelines, and fundamentally diverge from assumptions in standard active learning (e.g., static unlabeled pools) and lifelong learning (e.g., task boundaries, replay buffers). OWL-GPS thus demands a new class of adaptive, non-revisitable, and policy-driven learning strategies.

Our key contributions are summarized as follows:


\begin{itemize}[noitemsep, topsep=0pt, leftmargin=*]
\item We define the \textbf{Open-World Learning for Geospatial Prediction and Sampling (OWL-GPS)} setting, characterized by policy-driven sampling, strict budgets, evolving distributions, and non-revisitable inputs.
\item We propose a \textbf{concept-guided relevance encoder} based on a Conditional Variational Auto-Encoder. It is conditioned on domain concept variables informed by domain expert collaboration.


\item We develop a \textbf{relevance-aware meta-training strategy} that forms diverse, high-utility meta-batches on the fly. By combining uncertainty and conceptual dissimilarity, it supports adaptation in spatially continuous, non-stationary settings without relying on static-buffers or fixed episodes. 
\item We validate our framework on a \textbf{real-world discovery task}, \emph{PFAS hotspot detection} using a \textbf{novel dataset}, and a \textbf{generalization experiment}, \emph{rare land cover identification}, demonstrating robustness under sparse supervision and distribution shift, setting an OWL-GPS benchmark.
\end{itemize}

.

\vspace{-16pt}
\vspace{-9pt}
\section{Problem Formulation}
\label{sec:problem-formulation}
\vspace{-4pt}
Driven by real-world challenges, we propose the novel problem setting termed OWL-GPS. OWL-GPS captures the need to balance constrained data collection during training with effective decision-making at deployment. Specifically, we consider a scenario where, during training, we operate under a limited budget to acquire labeled data, aiming to learn a model that generalizes well across diverse geospatial regions. 
At inference time, a separate query budget is available; the objective expands to not only improving the model but also efficiently identifying regions with target presence.

\begin{wrapfigure}{r}{0.50\textwidth}
    \vspace{-10pt}
    \centering
    \includegraphics[width=\linewidth]{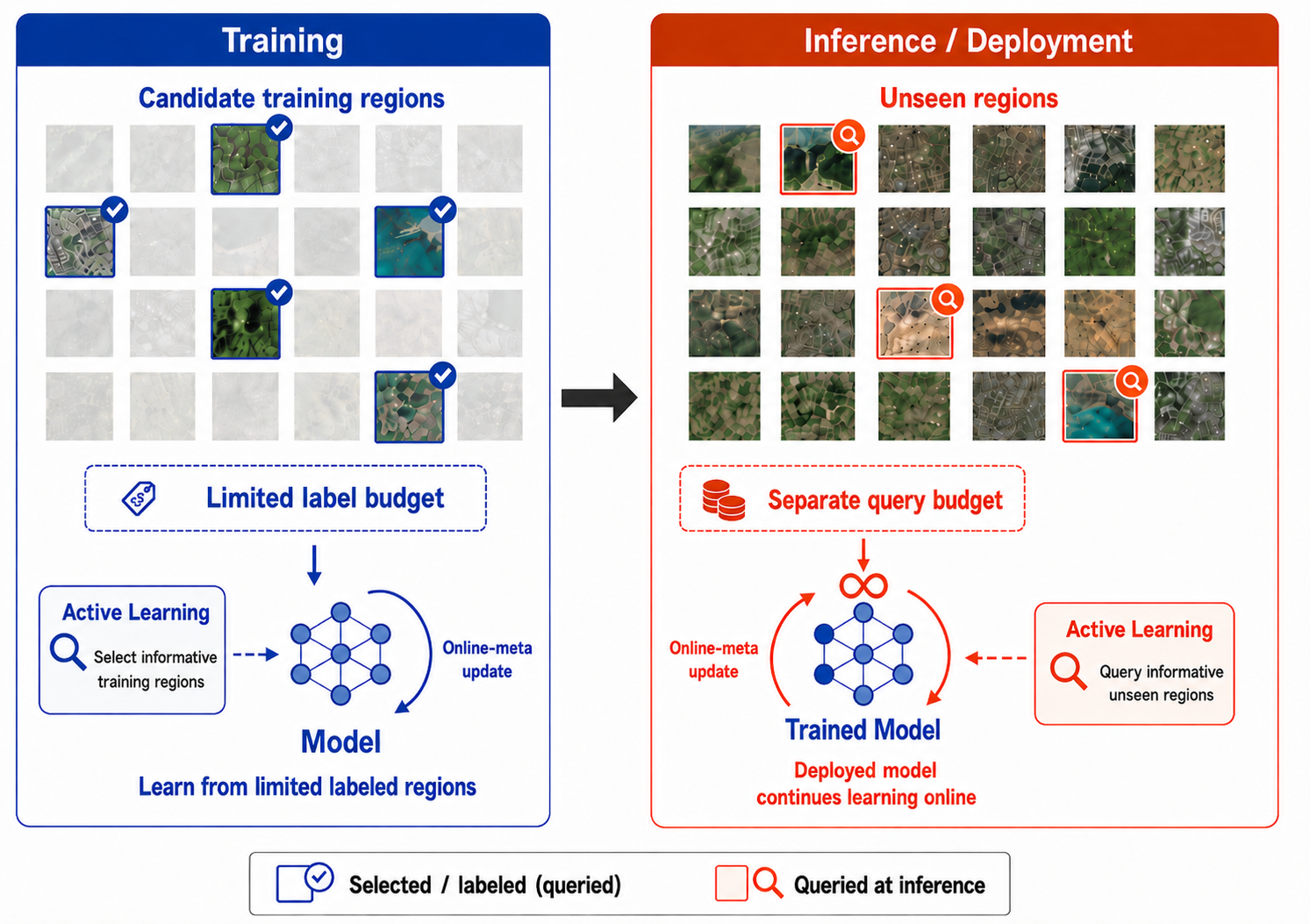}
    \vspace{-8pt}
    \caption{\small Problem Formulation.}
    \label{fig:problem_formulation}
    \vspace{-10pt}
\end{wrapfigure}

This distinction introduces a core challenge: how to select queries during inference that both respect the query budget and maximize discovery of the target class. This is further exacerbated in real-world applications, where data distributions may differ significantly across regions due to domain shifts in environmental, spatial, or temporal factors. We formalize this as follows. Let $X_{train} = \{ x_{train}^{(1)}, \ldots, x_{train}^{(N)} \}$ and $X_{test} = \{ x_{test}^{(1)}, \ldots, x_{test}^{(M)} \}$ denote two sets of images covering diverse geospatial regions. Typically, $M > N$, as training samples are usually gathered from localized regions before the model is deployed in an open-world setting. Each image $x^{(j)}$ is composed of visual features that capture spatial and environmental factors. The $k$-th pixel in $x^{(j)}$, denoted $x_k^{(j)}$, corresponds to a binary label $y_k^{(j)} \in \{0, 1\}$ indicating the presence or absence of the target.

At each query step $t$, a policy $\pi_{\theta_{t-1}}$ selects a region $q_t$ from the unobserved set to query. After querying, the true labels $y^{(q_t)}$ are revealed for the selected region $x^{(q_t)}$, and predictions are made by the model using the current parameters $\theta_{t-1}$. The objective is to select a sequence of queries that maximizes the number of correctly identified target pixels (true positives) across all queried regions, subject to a total cost budget $\mathcal{C}$.
\vspace{-8pt}

\begin{equation}
\label{E:objective}
\begin{aligned}
\max_{\{\pi_{\theta^{t-1}}\}} \quad
& \sum_{t} \sum_{i=1}^{P}
y_i^{(q_t(\pi_{\theta^{t-1}}))} \;
\mathbf{1}\Big[
y_i^{(q_t(\pi_{\theta^{t-1}}))} = 
& \qquad 
\underbrace{
\big[
\pi_{\theta^{t-1}}\big(
x^{(q_t(\pi_{\theta^{t-1}}))}
\big)
\big]_i
}_{\text{i-th pixel predicted by model given } q_t(\pi_{\theta^{t-1}})}
\Big] \\
\text{s.t.} \quad
& \sum_{t \ge 0}
c\!\left(
q_{t-1}(\pi_{\theta^{t-2}}),
q_t(\pi_{\theta^{t-1}})
\right)
\le \mathcal{C}
\end{aligned}
\end{equation}

\vspace{-8pt}
Here $c(i, j)$ is the cost associated with uncovering the ground-truth level $y^{(j)}$ associated with $x^{(j)}$ when initiating the query process from the region associated with $x^{(i)}$, and $P$ is the total number of pixels in a given region $x$. The indicator function ensures that only the correct target detections, corresponding to the binary label 1, are counted as success. The constraint enforces that the cumulative query cost does not exceed the total allowed budget.
Intuitively, this objective captures the dual challenge at the heart of OWL-GPS: making the best use of a limited number of queries to discover target regions, despite potential domain shifts and training-data limitations. To the best of our knowledge, the OWL-GPS problem setting we propose is novel, thereby necessitating the development of a novel framework specifically designed to address its unique challenges.

\vspace{-9pt}

\section{Related Work}
\paragraph{Environmental Monitoring and PFAS Modeling.}
Environmental contamination, such as PFAS, is typically measured through laboratory-based analysis (e.g., LC-MS \cite{Shoemaker_2020_Method537}), which is accurate but costly and spatially sparse. Prior work has explored hydrological simulation models (e.g., SWAT \cite{SWAT2023}, MODFLOW \cite{ktorcoletti_2012}), geostatistical interpolation methods such as kriging \cite{esri_kriging}, and heuristic approaches based on proximity to known contamination sources \cite{Salvatore_Mok_Garrett_2022}. While effective in static prediction settings, these methods often rely on strong assumptions or require extensive data and computation, limiting their scalability in large, heterogeneous environments. In particular, such approaches typically assume access to a fixed dataset or full spatial coverage, and operate via offline computation rather than sequential decision-making. As a result, they do not naturally extend to the OWL-GPS setting, where data must be acquired adaptively under strict budget and non-revisitable constraints. Our approach provides a data-driven alternative that incorporates domain priors while enabling adaptive, budget-constrained discovery.
\paragraph{RL-based Approaches for Active Geospatial Search (AGS)} 
AGS frameworks, which apply RL to optimize exploration strategies, have been explored in several related geospatial settings~\citep{sarkar2023partially,Sarkar_DiChristofano_Das_Fowler_Jacobs_Vorobeychik_2024,Sarkar_Lanier_Alfeld_Feng_Garnett_Jacobs_Vorobeychik_2024}. These methods typically depend on large-scale labeled datasets to learn effective exploration strategies by simulating search episodes. Additionally, these approaches are often confined to narrow, spatially localized, task-specific settings, such as detecting particular objects (e.g., cars) in satellite imagery.  
In contrast, our work tackles the challenge of geospatial search with limited data in an open-world context, in which the search is spatially broad, and we must strategically acquire data as well as learn to identify the target of interest -  which may be a diffuse phenomenon such as water pollution - via policy learning. Broader discussions are in the Appendix. 
\vspace{-8pt}
 \paragraph{Active, Online, Meta, and Lifelong Learning}
Active learning selects informative samples to reduce labeling cost~\citep{ren2021survey,cacciarelli2023active}, and online learning enables continuous updates as new data arrives~\citep{hoi2021online,shalev2025online}. However, both typically assume stationary task distributions, limiting their effectiveness in dynamic geospatial settings. Meta-learning methods~\citep{finn2017model,gharoun2024meta} rely on a predefined meta-training set to enable rapid policy adaptation, an assumption that does not hold in our setting.
Lifelong learning (LL)~\citep{kirkpatrick2017overcoming,shin2017continual,vogelstein2025simple} addresses non-stationarity via replay, task boundaries, or regularization. Recent advances have begun addressing more constrained settings: methods such as ~\citep{jung2020adaptive} avoid replay buffers; task-free and prototype-based methods like those by ~\citep{aghasanli2025prototype} operate without explicit task labels or large buffers. However, these approaches still assume either static selection pools, or environments where sampled examples remain accessible. In contrast, OWL-GPS imposes strict sampling budgets, prohibits revisitation, and requires policy-driven selection, rendering LL methods incompatible without major adaptation.

\vspace{-8pt}
\section{Methodology}
\label{sec:method}
\begin{figure} 
\centering 
\includegraphics[width=0.85\textwidth]{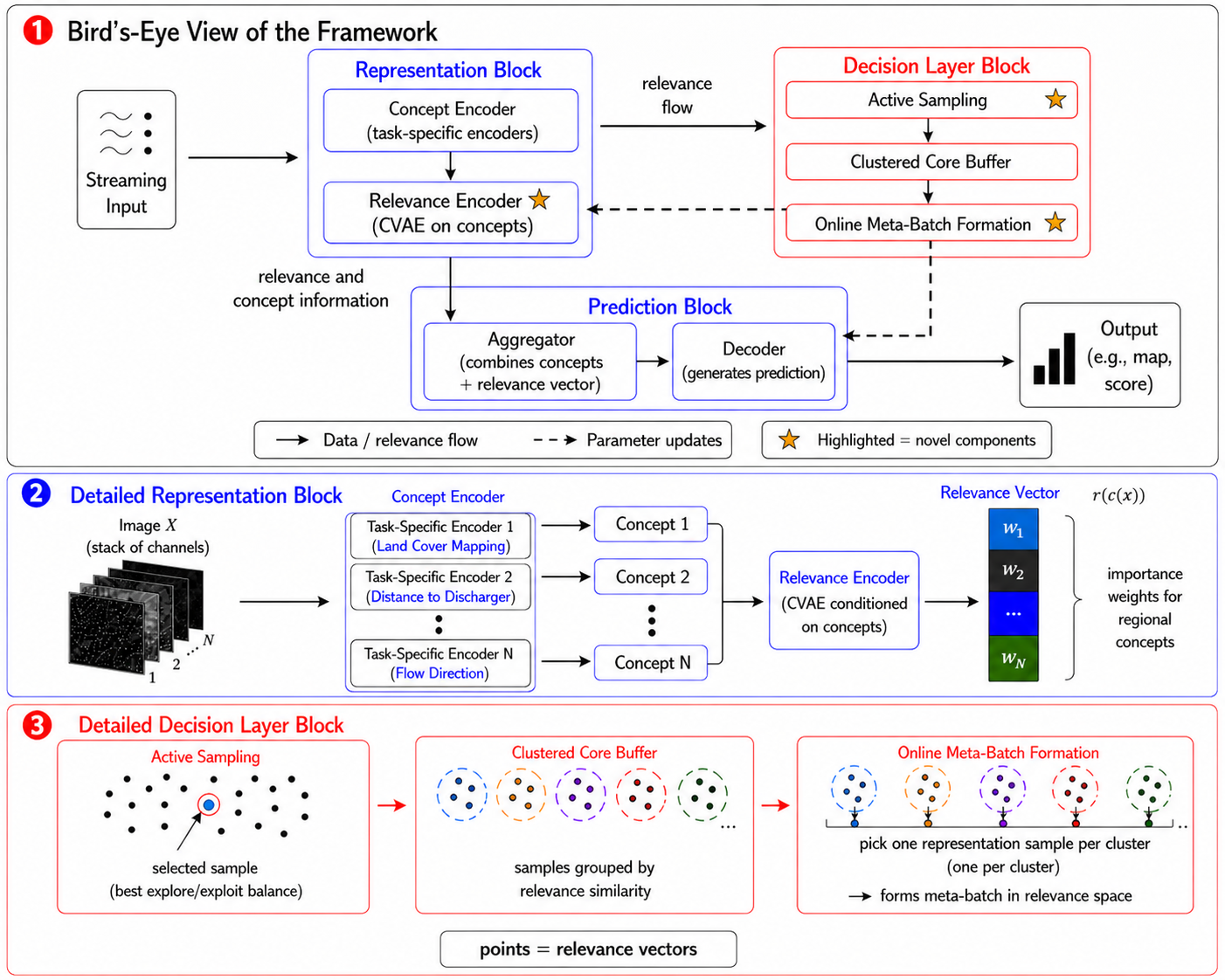} 
\caption{\small{Diagram of Framework}} \label{fig:yourlabel6} 
\vspace{-4pt} 
\end{figure}

We introduce a relevance-driven learning framework for the OWL-GPS setting, where a concept-conditioned latent variable governs prediction, sampling, and adaptation under strict budget and non-revisitable constraints. 

We learn a relevance vector that captures how domain-specific concepts influence target presence in a region, and use it as a unifying decision variable. In particular, relevance directly governs (i) sampling decisions, (ii) meta-training set formation, and (iii) online adaptation, while also supporting prediction through the decoder.

\vspace{-4pt}

\subsection{Representation Layer: Learning Concept and Relevance}

\paragraph{Concept Encoder} 
To support relevance-driven decision-making, we first construct structured concept representations encoding domain-specific factors influencing target presence, which are known from prior studies but vary across regions and tasks. Each factor is modeled as a latent variable within a generative framework~\citep{Blumenfeld_2023,kingma2013auto}, yielding low-dimensional embeddings that capture underlying geospatial structure. This latent representation provides a common space for comparing regions and selecting informative samples despite heterogeneous and spatially correlated inputs.

We pre-train an autoencoder-style model to learn latent concept representations $\tilde{c}_k(x)$ for each of the $K$ factors characterizing a region $x$, forming a concept vector $\tilde{c}(x) = [\tilde{c}_1(x), \dots, \tilde{c}_K(x)]$. The number of concept channels $K$ varies across datasets (typically $K\approx40$), capturing domain factors such as land cover, source proximity, and hydrological structure (see Appendix). 

To promote diversity and reduce redundancy, we apply Gram–Schmidt orthogonalization (GS~\citep{leon2013gram}) to obtain $c(x) = GS(\tilde{c}(x))$. These concept vectors serve as structured inputs to the relevance model and guide downstream prediction and sampling. Implementation details and additional analyses are provided in the Appendix.

\vspace{-8pt}
\paragraph{Relevance Encoder and Decoder} 
Given a concept vector $c(x)$, we introduce a \emph{concept-conditioned relevance variable} $r(c(x))$ that captures how each concept influences target presence in region $x$. Since relevance is unobservable and varies across regions, we model it as a latent variable using a conditional variational autoencoder (CVAE)~\citep{kingma2013auto,sohn2015learning}.

Specifically, we compute:

\begin{equation}\label{eq:obj}
\small
\begin{aligned}
\pi_{\theta^0}\big(y \mid c(x)\big)
&= p\big(y \mid c(x)\big)
&= \mathbb{E}_{r\big(c(x)\big) \sim p_{\zeta^0}\big( r(c(x) \big)} \bigg[p_{\phi^0}\big(y \mid c(x), r(c(x)\big) \bigg]
\end{aligned}
\end{equation}

Here, the policy parameter $\theta^0$ consists of a relevance encoder parameterized by $\zeta^0$, along with a decoder parameterized by $\phi^0$ that predicts $y$. Here, $r\big(c(x)\big) \in \mathbb{R}^K$ denotes the relevance vector and is represented as $r\big(c(x)\big) = \mu_{\zeta^0}(c(x)) + \epsilon \cdot \sigma_{\zeta^0}(c(x)) = [\alpha_{c_1(x)}, \alpha_{c_2(x)}, \ldots, \alpha_{c_K(x)}]$, where $i$'th component, i.e., $\alpha_{c_i(x)}$, quantifies the contribution of the $i$-th concept to the presence of the target in region $x$. Here,  $\mu_{\zeta^0}$ and $\sigma_{\zeta^0}$ denote the mean and standard deviation of the approximate posterior distribution of relevance, with $\epsilon \sim \mathcal{N}(0, I)$. 
Importantly, this relevance vector acts as a shared representation that governs prediction, sampling, and adaptation.
We present the implementation details of the relevance encoder and decoder, along with a comparison to simpler alternatives against our formulation, in the Appendix.

\begin{proposition}\label{met:prop_elbo}
Optimizing ~\ref{eq:obj} is equivalent to minimizing the following objective (derivation in appendix):
\begin{align*}
\small
\min_{\theta=(\phi,\zeta)}\; &
\mathbb{E}_{r \sim p_{\zeta^0}(r( c(x^{(j)})))}
\big[
\log p_{\phi^0}(y^{(j)} \mid c(x^{(j)}), r)
\big]
&\quad - \mathrm{KL}\!\left(
p_{\zeta^0}(r(c(x^{(j)}))) \,\|\, p(r(c(x^{(j)})))
\right)
\end{align*}
\end{proposition}

\subsection{Decision Layer: Relevance-Guided Learning}

\subsubsection{Step 1: Relevance-Aware Meta-Batch Formation}

To enable adaptation under non-revisitable constraints, we construct meta-training batches in the relevance space, where each sample is represented by its relevance vector. This allows selection based on semantic importance while promoting diversity under distribution shift.

We maintain a fixed-capacity \textit{core buffer} of queried samples, assigning each a score $\frac{\text{duration}}{\text{count}+1}$, where \textit{duration} denotes time since insertion and \textit{count} tracks usage in meta-training. To form diverse batches, we cluster samples in the relevance space and select one representative per cluster:



\begin{equation}
\small
\begin{aligned}
\mathcal{D}^{\text{core}}_t
&= \left\{
x_c^* \;\middle|\;
x_c^* = \arg\max_{x \in \mathcal{E}_c} \right. 
\exp\!\left(
\frac{\mathrm{duration}_x}{\mathrm{count}_x + 1}
\right)\Big\}, 
&\qquad \left.
\; \forall c \in \{1,\dots,N_c \right\}
\end{aligned}
\end{equation}
Here, $N_c$ denotes the total number of clusters, $\mathcal{E}_c$ represents the set of elements belonging to cluster $c$, and $\mathcal{D}^{\text{core}}_t$ is the set of samples selected from the core buffer for meta-training at time $t$.


When the core buffer reaches capacity, samples exceeding a lifespan threshold are moved to a  \textit{reservoir buffer}, allowing previously discarded samples to be reconsidered as the model evolves. We sample $K$ elements from the $N$ samples from the reservoir as:

\begin{equation}
\small{\mathcal{D}^{\text{reservoir}}_t \sim \text{Sample}_K^{\text{w/o replace}}\left(\left\{\frac{\exp\left(\frac{\text{duration}_i}{\text{count}_i + 1}\right)}{\sum_{j=1}^N  \exp\left(\frac{\text{duration}_j}{\text{count}_j + 1}\right)}\right\}_{i=1}^{N}\right)}
\end{equation}

Analogous to the core buffer, the reservoir buffer enforces a lifespan-based eviction policy upon reaching capacity. The final meta-training batch combines both buffers: $\mathcal{D}_t = \mathcal{D}_t^{\text{core}} \cup \mathcal{D}_t^{\text{reservoir}}$, balancing reuse of informative samples with continued exploration in the relevance space.

\subsubsection{Step 2: Relevance-Driven Sampling}
\vspace{-3pt}
\paragraph{Relevance-Driven Sampling During Training}
The central objective during training is to sample data that most effectively supports the learning of the search policy. We propose a relevance-aware sampling strategy that, unlike standard uncertainty-based approaches~\citep{cacciarelli2023active}, prioritizes both uncertain and informative regions in the relevance space.

At each query step, we select:
\begin{multline}\label{eq:sampling_training}
\small
x^{(q_t)} =
\operatorname*{arg\,max}_{x \in X_{\text{train}} \setminus \{k\}}
\underbrace{\exp^{(w_x)}}_{\textit{relevance uncertainty}}
\sum_{i=1}^{P}
\underbrace{\exp^{\bigl(-\lvert \pi_{\theta^{t-1}}(x_i) - 0.5 \rvert\bigr)}}_{\textit{prediction uncertainty}}
\\\text{where } 
w_x = \sum_{x' \in \{k\}}
\lVert \mu_x - \mu_{x'} \rVert_2^2; k = \{ x^{(q_1)},\ldots, x^{(q_{t-1})}\} .
\end{multline}

Here, $\pi_{\theta^{t-1}}(x_i)$ denotes the predicted probability for pixel $i$, $P$ denotes the total number of pixels within region $x$, and $\mu_x$ is the mean of the relevance vector ($r\big(c(x)\big)$) (with alternate formulations tested via ablation in Table~\ref{tab:ablation} and UCB baseline in Table~\ref{tab:2019_21_LC}). A mathematical justification motivating this objective is provided in the Appendix. The term $w_x$ captures dissimilarity in the relevance space, promoting selection of uncertain and diverse regions.

\vspace{-6pt}
\paragraph{Relevance-Driven Sampling During Inference}
During inference, our goal shifts to efficiently discovering target-rich regions under a fixed query budget. To do so, we select regions using a relevance-aware exploration–exploitation strategy.

We first compute an \textit{exploitation score}, which prioritizes regions likely to contain targets for an unobserved region $x$ at step $t$ as:
\begin{equation}\label{eq:exploit_score}
    \small{\text{Exploit}_{ \pi_{\theta^{t-1}}}^{\mathrm{score}}(x) = \exp^{-w_x} \cdot \sum_{i=1}^{i=P} \exp^{\big(\pi_{\theta^{t-1}}(x_i) \big) }}
\end{equation}

Next, we define an \textit{exploration score}, which promotes both diversity and uncertainty:

\begin{equation}\label{eq:explore}
    \small{\text{Explore}_{ \pi_{\theta^{t-1}}}^{\mathrm{score}}(x) = \exp^{w_x} \cdot \sum_{i=1}^{i=P} \exp^{(- |\pi_{\theta^{t-1}}(x_i) -0.5|) }}
\end{equation}


To guide selection, we combine these scores using a budget-aware trade-off:
\begin{equation}\label{eq:final_score}
\small
\mathrm{Score}_{ \pi_{\theta^{t-1}}}(x) = \kappa(\mathcal{C}) \cdot \mathrm{Explore}_{ \pi_{\theta^{t-1}}}^{\mathrm{score}}(x) + (1 - \kappa(\mathcal{C})) \cdot \mathrm{Exploit}_{ \pi_{\theta^{t-1}}}^{\mathrm{score}}(x)
\end{equation}
Here, $\kappa(C)$ is a design choice that decreases over time to emphasize exploration early on and shift toward exploitation as the budget depletes. A simple linear form, $\kappa(\mathcal{C}) = \frac{\mathcal{C} - t}{\mathcal{C} + t}$, is used in practice. 


After computing the scores (as defined in~\ref{eq:final_score}) for all candidate regions, we select the region with the highest score for querying. We provide pseudocode in the Appendix. We provide a mathematical justification motivating the exploitation score formulation in~\ref{eq:exploit_score} in the Appendix.

\vspace{-4pt}
\subsubsection{Step 3: Online-Meta Update} 
To enable continual adaptation under sequential and non-revisitable constraints, we adopt an online meta-learning formulation ~\citep{finn2019online} where updates are performed using relevance-structured meta-batches.
At the $t$-th query step, the policy is updated as:

\begin{align}\label{eq:meta_online}
\tilde{\pi}_{\theta^{t-1}} &= \arg\min_{\pi_{\theta^{t-1}}} \; 
\mathcal{L} \left( 
    \pi_{\theta^{t-1}} - \nabla \mathcal{L}(\pi_{\theta^{t-1}}, \mathcal{D}_{t}^{\text{tr}}), 
    \mathcal{D}_{t}^{\text{te}} 
\right) \notag \\
\pi_{\theta^{t}} &= \arg\min_{\tilde{\pi}_{\theta^{t-1}}} \; 
\tilde{\pi}_{\theta^{t-1}} - \nabla \mathcal{L}\left( 
    \tilde{\pi}_{\theta^{t-1}}, (x^{(q_t)}, y^{(q_t)}) 
\right)
\end{align}
Here, $\mathcal{D}_{t}$ denotes the meta training batch constructed in the relevance space. 
We randomly split $\mathcal{D}_{t}$ to form the training and evaluation sets, i.e., $\mathcal{D}_{t} = \mathcal{D}_{t}^{\text{tr}} \cup \mathcal{D}_{t}^{\text{te}}$. The loss $\mathcal{L}$ represents pixel-wise focal loss \cite{Lin_Goyal_Girshick_He_Dollár_2018}.

Following each query, the newly observed sample $(x^{(q_t)}, y^{(q_t)})$ is incorporated into the core buffer, with eviction governed by the lifespan constraint. This update mechanism allows the policy to adapt while leveraging semantically structured information captured by the relevance representation.

\begin{proposition}
Relevance-Aware Meta-batch $\mathcal{D}_t$ facilitates a faster convergence:
$$\mathbb{E}[\mathcal{L}(\theta_T)] - \mathcal{L}(\theta^*) \leq \mathcal{O}\left( \frac{\sigma_{rel}^2}{T} + \frac{D^2}{T^2} \right)$$
Here, the distance from the initial parameters $\theta_1$ to the optimal parameters $\theta^*$ is bounded by $D$: $\left\| \theta_1 - \theta^* \right\|^2 \le D^2$, $\sigma_{rel}^2$ = $||\nabla \mathcal{L}_{\mathcal{D}_T}(\theta_T) - \nabla \mathcal{L}(\theta^*)||$ , under the relevance-aware strategy (Detailed in Appendix).
 Here $\theta^*$ is the optimal parameter in hindsight: $\theta^* = \arg\min_\theta \sum_{t=1}^T \mathcal{L}_t(\theta)$.

\end{proposition}

\vspace{-2pt}

\vspace{-6pt}
\section{Experiments}
\vspace{-5pt}
\paragraph{Evaluation Metrics}
Since OWL-GPS seeks to maximize the identification of locations containing the target of interest, we assess performance using a \emph{Success Rate (SR)} metric as defined below (assuming uniform query cost, i.e., $c(i,j)= 1$):
\vspace{-2pt}
\begin{equation}
\small
\begin{aligned}
\mathrm{SR}
&= \frac{1}{\mathcal{C}}
\sum_{t=1}^{\mathcal{C}} \frac{1}{\min\{\mathcal{C}, U_t\}}
   \sum_{i=1}^{P}
   y_i^{(q_t(\pi_{\theta^{t-1}}))}
&\quad \cdot \mathbf{1}\!\left[
   y_i^{(q_t(\pi_{\theta^{t-1}}))} 
   =
   \left[\pi_{\theta^{t-1}}\!\left(
   x^{(q_t(\pi_{\theta^{t-1}}))}
   \right)\right]_i
\right]
\end{aligned}
\end{equation}
Here, $U_t$ denotes the maximum number of target pixels in the image queried at step $t$ by the policy $\pi_{\theta_{t-1}}$. The SR metric evaluates how effectively the model selects regions that truly contain the target of interest. Specifically, it computes the proportion of selected regions that are both predicted and \emph{actually} confirmed to be target regions, normalized by the available number of targets in each queried image. This ensures that the metric accounts not just for correct predictions, but also for the \emph{efficiency of query selection} under a constrained sampling budget. SR is particularly appropriate for the OWL-GPS task, where the overarching objective is to discover as many distinct target regions as possible with limited queries. 

To complement SR, we report \emph{accuracy, F-score, precision}, and \emph{recall} to assess pixel-level prediction quality. We evaluate (i) spatial generalization within 2019 and (ii) temporal generalization from 2019 to 2022, based on labels availability across years. These settings reflect real-world monitoring with sparse, costly labels and distribution shift, rather than standard ML benchmarks. Dataset construction, preprocessing, training setup, compute, and code details are provided in the Appendix.


\vspace{-8pt}
\paragraph{Baselines:}
We compare against 6 baselines spanning greedy search, active learning, bandit-based selection, and meta-learning (Table~\ref{tab:baselines}). These are selected to reflect diverse modeling assumptions under our OWL-GPS constraints. We also include a domain-informed greedy baseline based on pollutant transport simulation \cite{SWAT2023} (experimental setup in appendix) to reflect how domain knowledge may be used without learning.  Detailed discussions on baselines are in Appendix. 
\footnote{\footnotesize{
Traditional active learning methods like BALD~\citep{houlsby2011bayesian}, CoreSet~\citep{sener2017active}, and BADGE~\citep{ash2019deep} assume static unlabeled pools and iterative retraining, making them unsuitable for streaming, budget-constrained, non-revisitable OWL-GPS without major adaptation. Our baselines cover OWL-GPS axes: greedy single-pass search (GA), adaptive exploration (UCB), constrained active learning (AL), and online/meta-learning for adaptation.}}
\begin{table}[t]
\centering
\small
\caption{Summary of baseline methods used for comparison.}
\label{tab:baselines}
\begin{tabularx}{\columnwidth}{p{1.2cm} X}
\toprule
\textbf{Baseline} & \>\>\>\>\>\>\>\>\textbf{Description} \\
\midrule
\small{\textbf{GA}} &
 \small{\textbf{G}reedy \textbf{A}pproach: A non-adaptive, single-pass method that queries the highest-confidence regions after supervised training.} \\

 \small{\textbf{DIG}} &
\small{\textbf{D}omain-\textbf{I}nformed \textbf{G}reedy: A single-pass baseline that leverages domain-informed simulation to estimate contamination and rank regions without learning or sequential adaptation.} \\

\small{\textbf{Prithvi}} &
\small{A standard supervised geospatial model that iteratively selects high-confidence samples with online updates during inference.~\citep{Blumenfeld_2023}.} \\

\small{\textbf{UCB}} &
\small{\textbf{OWL-GPS} with a bandit-based exploration and exploitation sampling policy using \textbf{U}pper \textbf{C}onfidence \textbf{B}ounds~\citep{auer2002finite}.} \\

\small{\textbf{AL}} &
\small{\textbf{A}ctive \textbf{L}earning: Exploration-focused, pure uncertainty based sampling during training and inference~\citep{cacciarelli2023active}.} \\

\small{\textbf{AML}} &
\small{\textbf{A}ctive \textbf{M}eta \textbf{L}earning: \textbf{OWL-GPS} with meta-batch updates driven by latent representations of previously queried samples~\citep{kaddour2020probabilistic}.} \\

\small{\textbf{OML}} &
\small{\textbf{O}nline \textbf{M}eta \textbf{L}earning: \textbf{OWL-GPS} with continual single-sample meta-adaptation, replacing episodic batch updates~\citep{finn2019online}.} \\
\bottomrule
\end{tabularx}
\end{table}
\begin{table}[t]
\centering
\scriptsize
\setlength{\tabcolsep}{3pt}
\caption{\footnotesize{Comparison with baselines (mean $\pm$ std. over 3 trials). DIG is a deterministic domain-informed static baseline for PFAS only.}}
\label{tab:2019_21_LC}

\begin{tabular}{p{0.06\columnwidth} p{0.12\columnwidth}
                p{0.13\columnwidth} p{0.13\columnwidth}
                p{0.13\columnwidth} p{0.13\columnwidth}
                p{0.13\columnwidth}}
\toprule
Year & Method & Acc. & F-score & Prec. & Recall & SR \\
\midrule
\multirow{7}{*}{LC}
 & GA & 46$\pm$0.8 & 35$\pm$1.0 & 48$\pm$0.9 & 40$\pm$1.3 & 47$\pm$1.4 \\
 & AL & 58$\pm$1.8 & 41$\pm$1.5 & 50$\pm$1.2 & 47$\pm$1.9 & 60$\pm$1.0 \\
 & Prithvi & 61$\pm$1.2 & 40$\pm$1.4 & 51$\pm$1.5 & 56$\pm$1.7 & 61$\pm$1.8 \\
 & UCB & 55$\pm$2.1 & 45$\pm$1.7 & 53$\pm$1.3 & 58$\pm$1.5 &  55$\pm$1.3 \\
 & AML & 91$\pm$1.3 & 48$\pm$1.3 & 46$\pm$1.6 & 50$\pm$1.8 & 90$\pm$1.1 \\
 & OML & 92$\pm$1.7 & 48$\pm$1.6 & 47$\pm$1.8 & 50$\pm$1.9 & 92$\pm$1.5 \\
 & \textbf{Ours} & \textbf{93$\pm$1.4} & \textbf{71$\pm$1.1} & \textbf{70$\pm$1.3} & \textbf{72$\pm$1.2} & \textbf{94$\pm$1.0} \\
\midrule
\multirow{8}{*}{2022}
 & DIG & 37 & 39 & 40 & 38 & 70 \\
 & GA & 72$\pm$1.6 & 46$\pm$1.3 & 43$\pm$1.5 & 50$\pm$1.4 & 77$\pm$2.2 \\
 & AL & 72$\pm$1.3 & 44$\pm$1.2 & 39$\pm$1.0 & 50$\pm$1.5 & 76$\pm$1.7 \\
 & Prithvi & 71$\pm$1.7 & 52$\pm$1.4 & 53$\pm$1.3 & 54$\pm$1.1 & 78$\pm$1.8 \\
 & UCB & 75$\pm$1.5 & 52$\pm$1.2 & 53$\pm$1.1 & 57$\pm$1.3 & 81$\pm$1.9 \\
 & AML & 68$\pm$2.0 & 46$\pm$1.5 & 45$\pm$1.7 & 48$\pm$1.4 & 80$\pm$2.1 \\
 & OML & 75$\pm$1.7 & 51$\pm$1.2 & 52$\pm$1.4 & 51$\pm$1.5 & 84$\pm$1.9 \\

 & \textbf{Ours} & \textbf{75$\pm$0.9} & \textbf{54$\pm$1.0} & \textbf{56$\pm$1.2} & \textbf{53$\pm$1.1} & \textbf{86$\pm$0.8} \\
\midrule
\multirow{8}{*}{2019}
 & DIG & 39 & 40 & 41 & 39 & 71 \\
 & GA & 84$\pm$1.5 & 46$\pm$1.3 & 42$\pm$1.4 & 50$\pm$1.7 & 92$\pm$2.0 \\
 & AL & 81$\pm$1.8 & 50$\pm$1.5 & 42$\pm$1.3 & 49$\pm$1.6 & 93$\pm$2.1 \\
 & Prithvi & 86$\pm$1.6 & 47$\pm$1.3 & 43$\pm$1.4 & 50$\pm$1.5 & 94$\pm$1.8 \\
 & UCB & 79$\pm$1.3 & 61$\pm$1.1 & 55$\pm$1.2 & 61$\pm$1.4 & 94$\pm$1.6 \\
 & AML & 86$\pm$2.0 & 46$\pm$1.4 & 43$\pm$1.5 & 50$\pm$1.3 & 95$\pm$2.0 \\
 
 & OML & 83$\pm$1.7 & 54$\pm$1.2 & 50$\pm$1.4 & 60$\pm$1.5 & 94$\pm$1.9 \\
 & \textbf{Ours} & \textbf{88$\pm$2.0} & \textbf{65$\pm$2.0} & \textbf{57$\pm$2.0} & \textbf{83$\pm$2.0} & \textbf{98$\pm$1.0} \\
\bottomrule
\end{tabular}
\end{table}
\vspace{-5pt}
\paragraph{Comparison with data from 2019} 
We first evaluate our method on the 2019 PFAS dataset~\citep{huerta2018presence}, covering diverse U.S. regions. As shown in Figure~\ref{fig:SR2019_21_LC}, our model progressively improves its Success Rate (SR), reflecting an increasing shift toward exploitation. Table~\ref{tab:2019_21_LC} confirms that our policy maintains superior performance.

\vspace{-4pt}
\paragraph{Uncovering a Specific Land Cover type within a Strict Sampling Budget}
To test generalization, we apply our framework to land cover (LC) data~\citep{dewitz2021national}, targeting the water class despite visually similar categories (e.g., ice). Using a sparsified dataset (details in Appendix), our model maintains SR above 90\% under tight supervision (Figure~\ref{fig:SR2019_21_LC}), and shows strong predictive accuracy (Table~\ref{tab:2019_21_LC}), demonstrating adaptability to low-data settings.

\paragraph{Comparison using data spanning 2019 and 2022}
\vspace{-8pt}

We further evaluate transferability across time by training on 2019 PFAS data and testing on 2022~\citep{USEPA2024}. Our method maintains relatively more stable performance across SR and predictive metrics (Table~\ref{tab:2019_21_LC}), validating its effectiveness in spatiotemporally evolving OWL-GPS settings.
\begin{wrapfigure}{r}{0.48\textwidth} 
\begin{subfigure}[b]{0.45\linewidth} 
\centering 
\includegraphics[width=\linewidth]{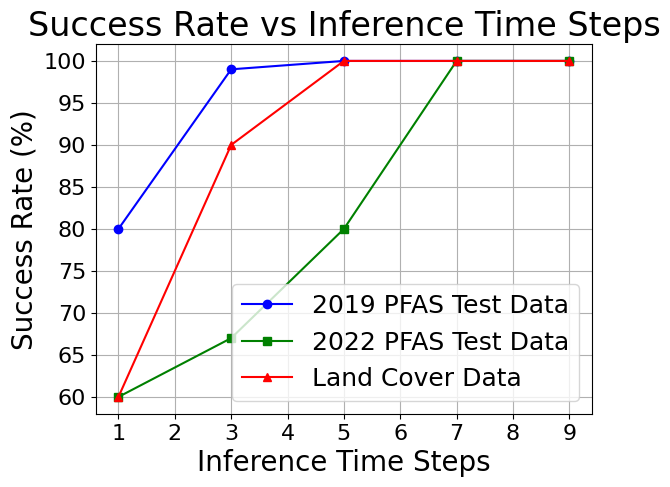} 
\caption{\small \emph{SR} over timesteps} 
\label{fig:SR2019_21_LC} 
\end{subfigure} 
\hfill 
\begin{subfigure}[b]{0.50\linewidth} 
\centering 
\includegraphics[width=\linewidth]{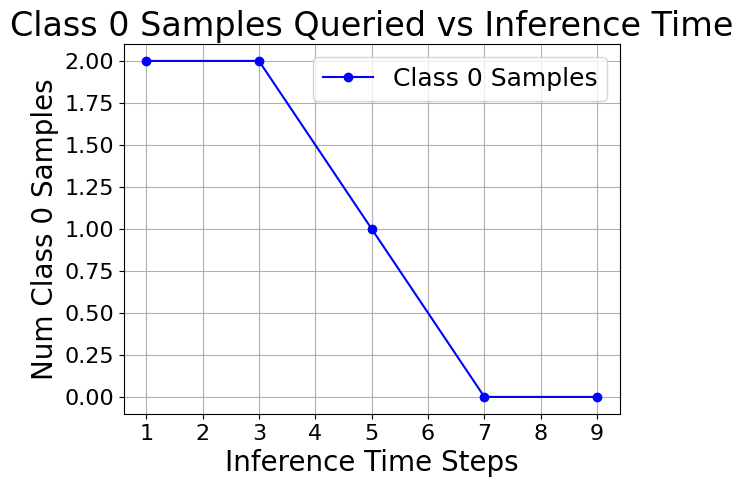}
 \caption{\small Exploration behavior} 
\label{fig:exp}
 \end{subfigure} 
\vspace{-4pt} 
\caption{\small Performance and exploration behavior on 2019 PFAS data.} 
\label{fig:sr_exploration} 
\vspace{-9pt} 
\end{wrapfigure}
\emph{Our approach consistently improves F-score across settings, supporting reliable positive-class discovery under costly false positives.} 
The 2022 setting is especially challenging as 2019-trained model is evaluated under strict budgets on later data with expanded channels and shifted conditions (Appendix).
\vspace{-4pt}
\paragraph{Ablations:}
For the following comparisons, we utilize PFAS data from 2019. Additional analyses are included in the Appendix to examine the sensitivity of our method to critical hyperparameters and design choices, such as buffer sizes, $\kappa(\mathcal{C})$, and varying sampling budgets. 
\begin{table}[t]
\centering
\small
\caption{Baseline weaknesses in the OWL-GPS setting.}
\label{tab:baseline_weaknesses}
\begin{tabularx}{\columnwidth}{p{1.2cm} X}
\toprule
\textbf{Baseline} & \textbf{Key Limitations in OWL-GPS setting} \\
\midrule
\textbf{UCB} &
Assumes independent arms; fails under spatial correlations, leading to inefficient exploration. \\

\textbf{AL} &
 Uncertainty-based sampling without relevance or feedback-driven correction; often selects informative samples yet fails to discover target-rich regions. \\

\textbf{DIG} &
Domain-informed approach based on pollutant transport simulation; lacks adaptation to newly observed samples, limiting sequential performance. \\

\textbf{GA} &
Exploits high-confidence regions early with no exploration; resulting in suboptimal performance. \\

\textbf{Prithvi} &
Statically trained with online updates at inference, but not designed for strategic exploration. \\

\textbf{AML} &
Relies on task identity and few-shot setup; fails in geospatial streams without task boundaries. \\

\textbf{OML} &
Updates on all samples without relevance filtering, resulting in noisy adaptation. \\
\bottomrule
\end{tabularx}
\end{table}
\vspace{-7pt}
\paragraph{Effectiveness of Relevance Encoder (RE)}
Removing the relevance encoder (No-RE) reduces F1 from 65\%→56\% (Table~\ref{tab:ablation}), confirming its role in structured prediction.

\vspace{-4pt}
\paragraph{Analyzing the Efficacy of Meta-training Set Formation}
Random meta-batch selection from a fixed-size core buffer degrades performance (Table~\ref{tab:ablation}), highlighting the importance of our relevance-based meta batch formation strategy.
\vspace{-2pt}
\paragraph{Probing the Impact of Relevance-Guided Sampling}
Removing relevance-guided sampling where sampling scores rely solely on the decoder’s output, and excluding all relevance-informed components (No-RG) significantly reduces SR and F1 (Table~\ref{tab:ablation}). 
\vspace{1pt}
\vspace{-11pt}
\begin{table}[H]
    \centering
    \scriptsize
    \caption{\textbf{\small{Ablation Study: Importance of Method Components.}}}
    \vspace{1mm} 
    \begin{tabular}{l c c c c c}
        \toprule
        Variant & Accuracy & F-score & Precision & Recall & SR ($\mathcal{C}=50$) \\
        \midrule
        No Relevance Encoder (No-RE)         & 86$\pm$1.8 & 56$\pm$1.7 & 50$\pm$1.5 & 74$\pm$1.8 & 95$\pm$2.0 \\
        Random Meta-Training Set             & 67$\pm$2.0 & 56$\pm$2.1 & 57$\pm$1.6 & 46$\pm$1.7 & 69$\pm$1.4 \\
        No Relevance-Guided Sampling (No-RG) & 72$\pm$1.9 & 55$\pm$1.7 & 55$\pm$2.0 & 50$\pm$1.7 & 77$\pm$1.5 \\
        \midrule
        \textbf{Ours (Full Framework)}       & \textbf{88$\pm$2.0} & \textbf{65$\pm$2.0} & \textbf{57$\pm$2.0} & \textbf{83$\pm$2.0} & \textbf{98$\pm$1.0} \\ 
        \bottomrule
    \end{tabular}
    \label{tab:ablation}
    \vspace{-4mm}
\end{table}

\vspace{1pt}
\paragraph{Interpretation of Decision Making and Analyzing the Exploration Strategy}
We analyze how the learned policy balances exploration and exploitation, and how relevance supports decision-making. The policy initially samples broadly, often querying non-target regions, but later focuses on target-rich areas (Fig.~\ref{fig:exp}), indicating adaptation under budget constraints. 

Relevance vectors provide interpretable decision support (Fig.~\ref{fig:two_figs}): high-confidence correct predictions emphasize key PFAS drivers from prior work and by our domain collaborators, including landfills, airports, and military sites~\citep{US_EPA_2021}, while error cases shift toward weaker or more ambiguous proxies. This suggests meaningful domain structure.
Additional visualizations in the Appendix) show a transition from diffuse exploration to spatially coherent regions near hydrologically connected industrial sources.



\vspace{-2pt}
\begin{figure}[h]
    \centering
    \begin{subfigure}{0.48\columnwidth}
        \centering
        \includegraphics[width=\linewidth]{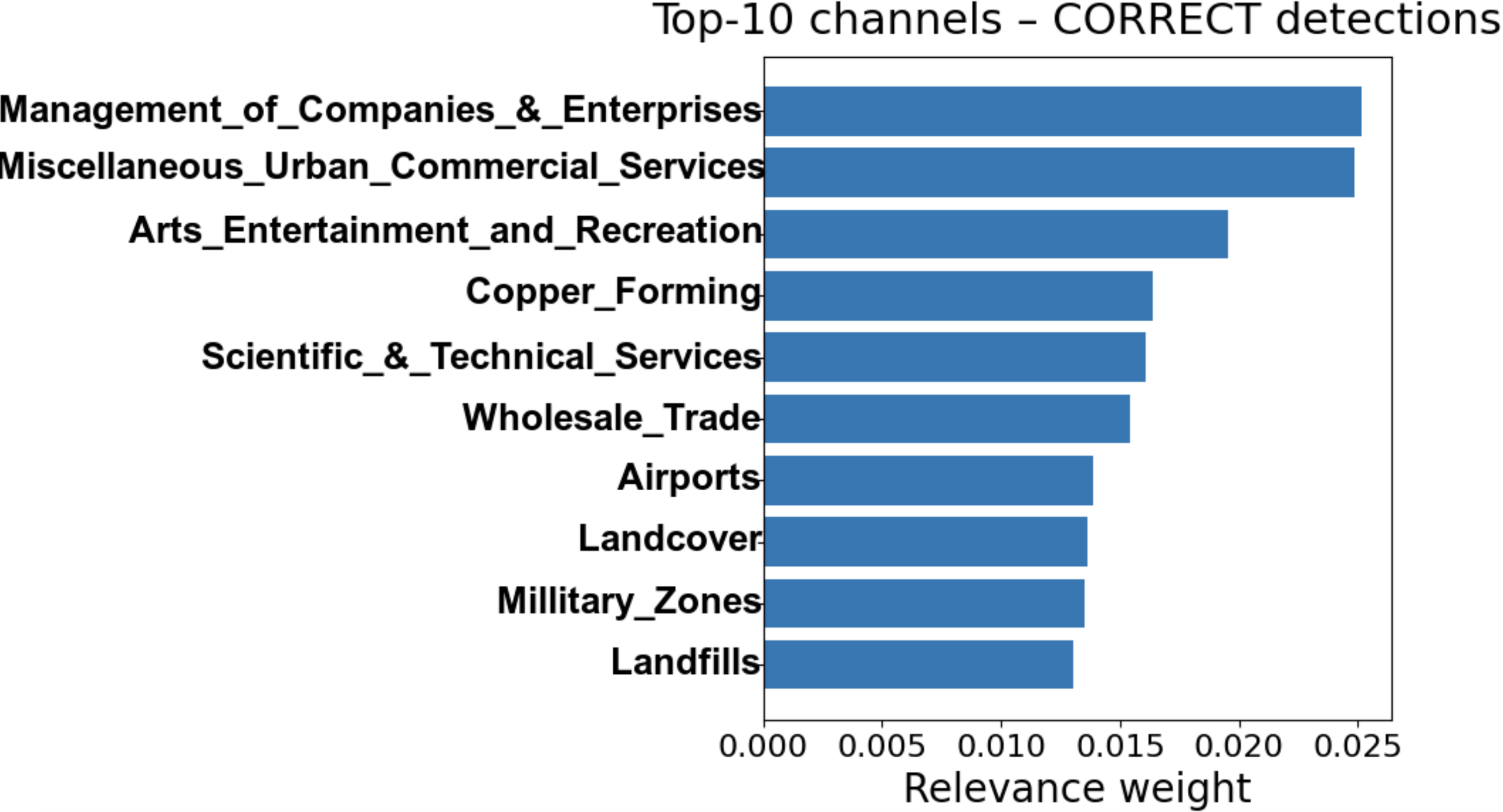}
        \caption{Top-10 channels of a correctly detected sample}
        \label{fig:sub1}
    \end{subfigure}\hfill
    \begin{subfigure}{0.48\columnwidth}
        \centering
        \includegraphics[width=\linewidth]{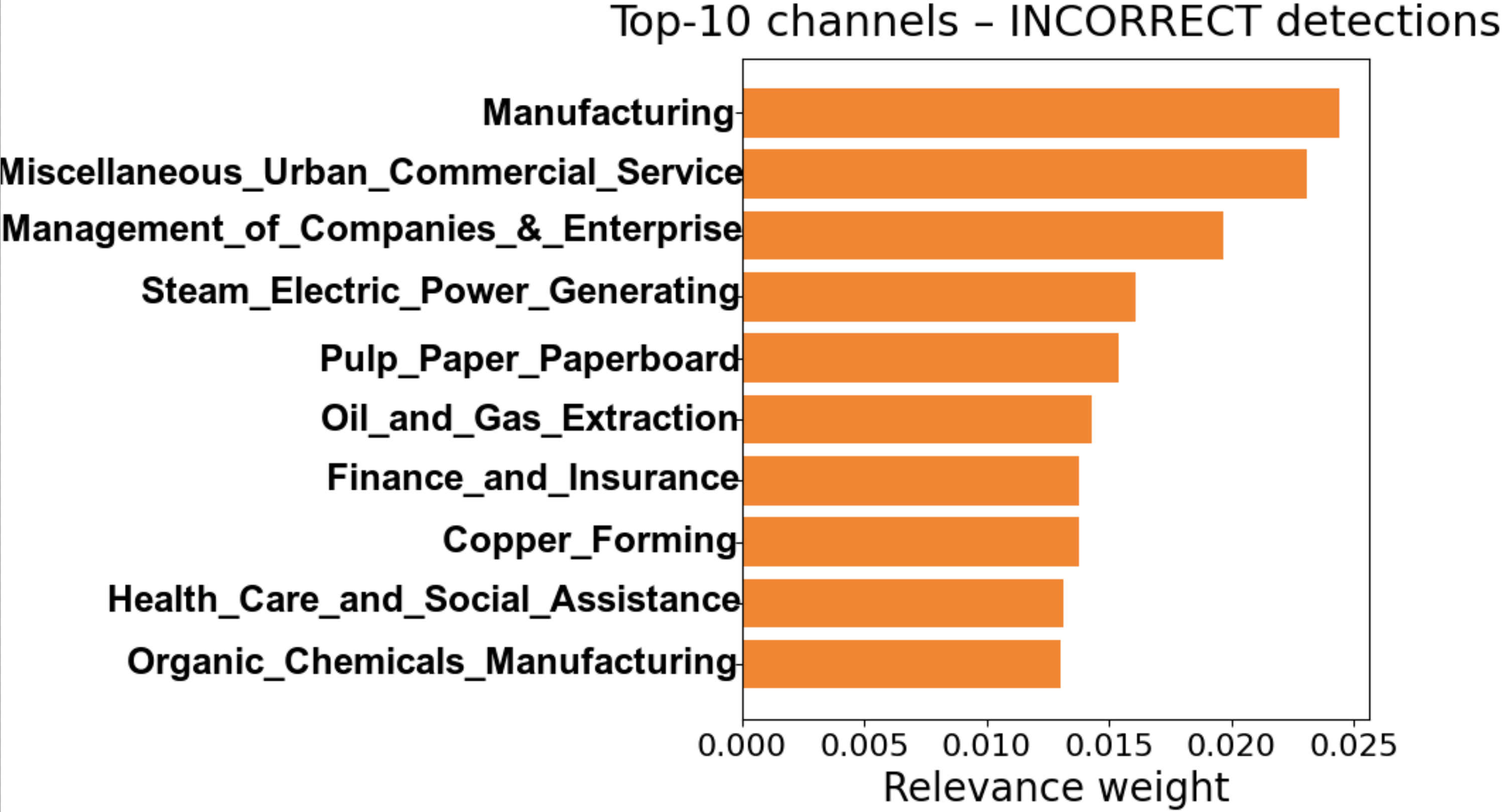}
        \caption{Top-10 channels of an incorrectly detected sample}
        \label{fig:sub2}
    \end{subfigure}

    \vspace{-4pt}
    \caption{\small{Interpreting decision-making through concept relevance.}}
    \label{fig:two_figs}
\end{figure}

\paragraph{Conclusion and Limitations}
We introduce a modular, interpretable framework for geospatial target discovery under strict sampling budgets, integrating entropy-based sampling with online meta-learning to adapt to sparse and evolving data. The approach scales well in real PFAS mapping and sparsified land-cover settings, showing strong performance in data-constrained environments. While it relies on domain-specific drivers, which may limit use in unstructured tasks, these drivers are common in many applications (e.g., pollution detection), making the framework broadly applicable.
\newpage

\nocite{langley00}

\bibliography{neurips_2026}
\bibliographystyle{plainnat}

\newpage
\onecolumn
\appendix

\newcommand{\ldotsfill}{\leavevmode\leaders\hbox to .5em{\hss.\hss}\hfill\kern0pt}

\section*{Adapting Actively on the Fly: Relevance-Guided Online Meta-Learning with Latent Concepts for Geospatial Discovery (Appendix)}
\vspace{3pt}

\section*{\>\>\>\>\>\>\>\>\>\>\>\>\>\>\>\>\>\>\>\>\>\>\>\>\>\>\>\>\>\>\>\>\>\>\>\>\>\>\>\>\>\>\>\>\>\>\>\>\>\>\>\>\>\>\>\>\>\>\>Overview of the Contents}

\noindent
\begin{tabularx}{\textwidth}{@{}Xr@{}}
\quad A. Effect of $\kappa(\mathcal{C})$ on Search Performance \ldotsfill & 16 \\ [0.5em]
\quad B. Connection between Active Sampling and Variance in Relevance Space \ldotsfill & 16 \\ [0.5em]
\quad C. Evolution of Discriminative Power in Relevance Space as Search Progress \ldotsfill & 16-17 \\ [0.5em]

\quad D: Proof of Proposition 4.1 \ldotsfill & 17-18\\ 

\quad E. Search Performance Across Varying Search Budget \ldotsfill & 18 \\ [0.5em]
\quad F. Proof of Proposition 4.2 \ldotsfill & 18-21 \\ [0.5em]
\quad G. On the Robustness of the Concept Encoder \ldotsfill & 21 \\ [0.5em]
\quad H. Influence of Core and Reservoir Buffer Sizes on Search Performance \ldotsfill & 21 \\ [0.5em]

\quad I. Theoretical Justification of the Training Sampling Objective \ldotsfill & 22-23 \\ [0.5em]
\quad J. Theoretical Justification of the Inference Sampling Objective \ldotsfill & 23-24 \\ [0.5em]

\quad K: Details of Greedy Intersection Clustering Algorithm \ldotsfill & 24-25 \\ [0.5em]

\quad L.  Additional Analysis on the Importance of Relevance Guided Sampling \ldotsfill & 25-26 \\ [0.5em]
\quad M. Importance of Orthogonalization in the Concept Space \ldotsfill & 26-27 \\ [0.5em]


\quad N: Dataset and Split Details \ldotsfill & 27 \\ [0.5em]

\quad O: Details of Pseudolabel Generation Procedure to induce Stability During Training \ldotsfill & 27-29 \\ [0.5em]

\quad P: Details of Training and Inference Hyperparameters \ldotsfill & 29 \\ [0.5em]

\quad Q: Architecture Details  \ldotsfill & 29-30 \\ 
\qquad Q.1  Details of Concept Encoder \ldotsfill & 29 \\ [0.5em]
\qquad Q.2 Details of Relevance Encoder \ldotsfill & 30 \\ [0.5em]
\qquad Q.3 Details of the Decoder \ldotsfill & 30 \\ [0.5em]
\qquad Q.4 Modeling Spatial and Hydrological Structure. \ldotsfill & 30 \\ [0.5em]

\quad R: Additional Interpretability Analysis of the proposed framework \ldotsfill & 30-31 \\ [0.5em]

\quad S: Additional Details about the Baselines \ldotsfill & 31-32 \\ [0.5em]

\quad T: Extended Related Work  \ldotsfill & 32-33 \\ 
\qquad T.1  Environment Monitoring \ldotsfill & 32 \\ [0.5em]
\qquad T.2 Geo-spatial Foundational Model \ldotsfill & 33 \\ [0.5em]

\quad U: Time and Memory Complexity Details of the Proposed Framework  \ldotsfill & 33-34 \\ 
\qquad U.1  Time Complexity Details  \ldotsfill & 33 \\ [0.5em]
\qquad U.2 Details about Compute Resource and Code \ldotsfill & 34 \\ [0.5em]

\quad V: Training and Inference Pseudocode \ldotsfill & 34-35 \\ [0.5em]

\quad W: Experimental Setup of Domain-Informed Greedy Baseline \ldotsfill & 35-36 \\ [0.5em]

\quad X: Impact Statement \ldotsfill & 36-37 \\ [0.5em]


\end{tabularx}

\newpage

\section{Effect of $\kappa(\mathcal{C})$ on Search Performance}



We conduct experiments to assess the impact of $\kappa(\mathcal{C})$ on the search policy's active search performance. Specifically, we investigate how amplifying the exploration weight, by setting $\kappa(\mathcal{C}) = \max{\{0, \kappa(\alpha \cdot \mathcal{C})\}}$ with $\alpha > 1$, and enhancing the exploitation weight by setting $\alpha < 1$, influence the overall effectiveness of the approach.
\begin{table}[!h]
  \vspace{-10pt} 
  \centering
  \caption{Effect of $\kappa(\mathcal{C})$}
  \label{tab: balance}
  \begin{tabular}{p{1.74cm}p{1.74cm}p{1.74cm}}
    \toprule
    \multicolumn{3}{c}{SR Performance across varying $\alpha$ with $\mathcal{C}=50$} \\
    \midrule
    $\alpha=0.2$ & $\alpha=1.0$ & $\alpha=5.0$ \\
    \midrule
     94\% $\pm$ 1.0 & \textbf{98\% $\pm$ 1.0} & 98\% $\pm$ 1.9 \\ 
    \bottomrule
  \end{tabular}
\end{table}
We present SR results for $\alpha \in \{ 0.2, 1, 5\}$ as shown in Table~\ref{tab: balance}. Performance improves when moving from an exploitation-heavy setting ($\alpha$ = 0.2) to balanced exploration-exploitation ($\alpha$ = 1.0). Increasing $\alpha$ further yields comparable performance within variance, indicating diminishing returns from overly aggressive exploration. These results suggest that while insufficient exploration degrades performance, the method remains robust to larger $\alpha$ values, with $\alpha$ = 1 providing a strong and stable operating point.





\section{Connection between Active Sampling and Variance in Relevance Space}
In this analysis, we aim to assess the effectiveness of the active sampling strategy. To achieve this, we compute the component-wise variance of the relevance vectors corresponding to the set of samples selected via the proposed active sampling strategy at two different query steps during the active sampling process at inference time. Our observations reveal a progressive decline in batch variance as the search advances, indicating a shift in the model’s behavior from exploration to increased exploitation. This behavior is illustrated in Figure~\ref{fig:figvar1},~\ref{fig:figvar2}, which highlights the effectiveness of the proposed sampling strategy. 

\begin{figure}[htbp]
  \centering
  \begin{minipage}[b]{0.45\textwidth}
    \centering
    \includegraphics[width=\linewidth]{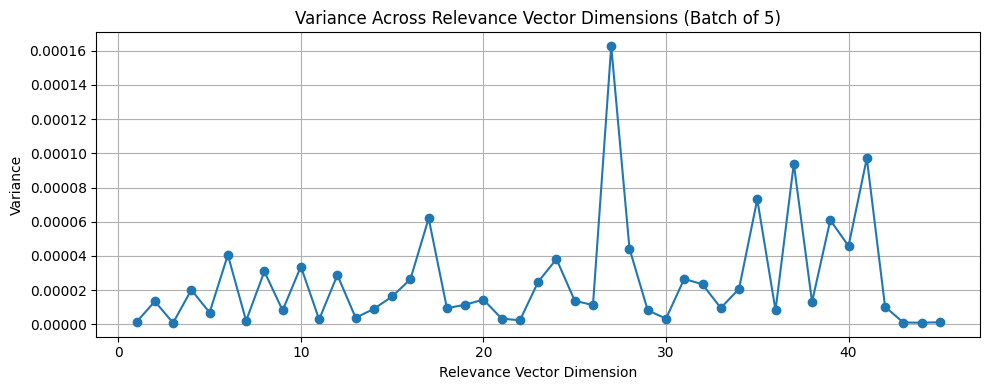}
    \caption{Variance across dimensions of the relevance vector during the initial Active Discovery Phase.}
    \label{fig:figvar1}
  \end{minipage}
  \hfill
  \begin{minipage}[b]{0.45\textwidth}
    \centering
    \includegraphics[width=\linewidth]{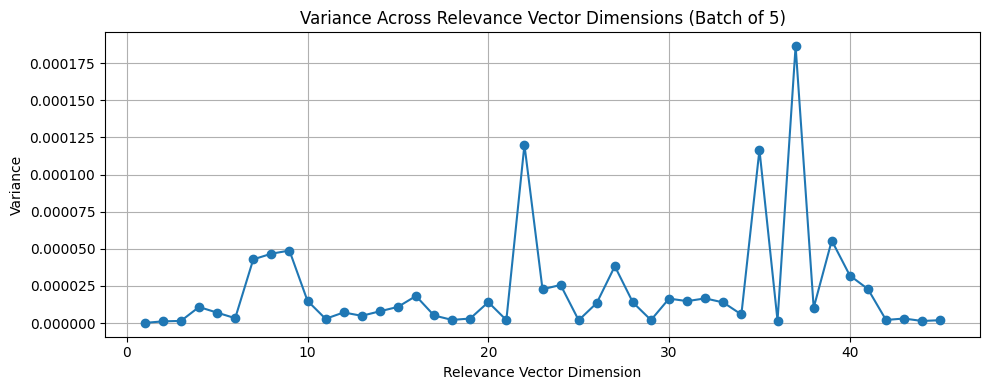}
    \caption{Variance across dimensions of the relevance vector during the later Active Discovery Phase.}
    \label{fig:figvar2}
  \end{minipage}
\end{figure}

\section{Evolution of Discriminative Power in Relevance Space as Search Progress}
To assess the effectiveness of our training paradigm, we investigate how the relevance space evolves throughout the active search process. We project the learned relevance vectors into two dimensions using t-SNE at two distinct training stages. Initially, the embeddings are densely entangled, showing little class-specific structure. Note that in our case, we only have two distinct classes, namely class 0 and class 1, where class 0 refers to the absence of the target region, and class 1 indicates the presence of the target region. However, as more informative samples are acquired through active querying, we observe a clear separation in embeddings 
corresponding to different input classes. This progressive disentanglement highlights the growing discriminative capacity of the relevance space, validating our approach's ability to refine representations in a data-efficient manner and ensuring increasingly reliable and robust search performance. We present the illustrative visualization in Figure~\ref{fig:fig1},~\ref{fig:fig2} using randomly selected two samples from two distinct classes. 

\begin{figure}[htbp]
  \centering
  \begin{minipage}[b]{0.45\textwidth}
    \centering
    \includegraphics[width=\linewidth]{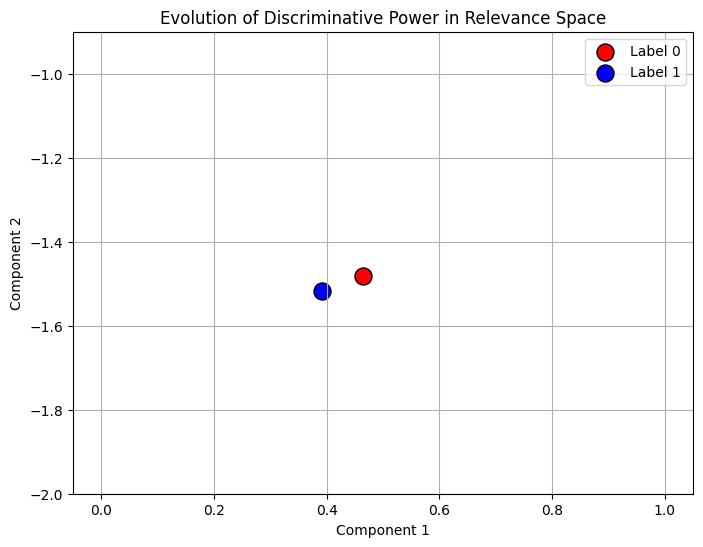}
    \caption{TSNE Visualization of Relevance Vectors during initial Active Discovery Phase.}
    \label{fig:fig1}
  \end{minipage}
  \hfill
  \begin{minipage}[b]{0.45\textwidth}
    \centering
    \includegraphics[width=\linewidth]{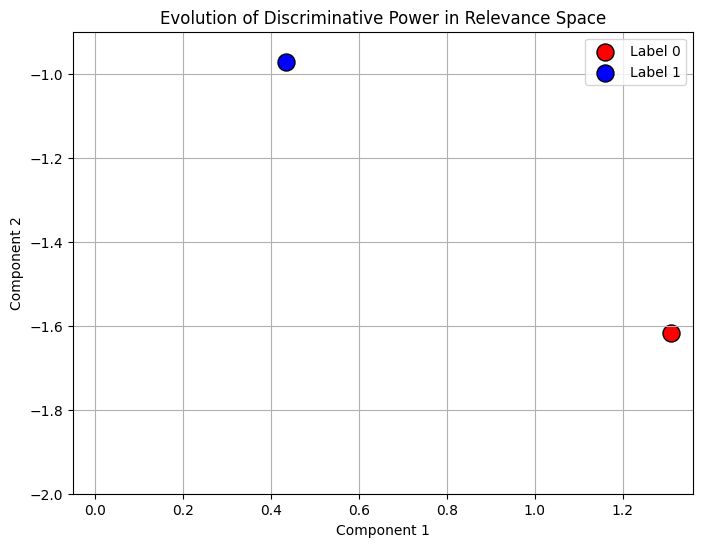}
    \caption{TSNE Visualization of Relevance Vectors during later Active Discovery Phase.}
    \label{fig:fig2}
  \end{minipage}
\end{figure}

\section{Proof of Proposition~\ref{met:prop_elbo}}\label{prop1}
\begin{tcolorbox}[colback=gray!10, colframe=gray!40!black, boxrule=0.5pt, arc=1mm]
Optimizing equation~\ref{eq:obj} is equivalent to minimizing the following objective:
\begin{align*}
\small
\min_{\theta = (\phi,\zeta)}\; & 
\mathbb{E}_{r\big(c(x^{(j)})\big) \sim p_{\zeta^0}\big( r(c(x^{(j)})) \big)} 
\left[ \log p_{\phi^0}(y^{(j)} \mid c(x^{(j)}), r(c(x^{(j)}))) \right] \\
& - \mathrm{KL} \left( p_{\zeta^0}\big( r(c(x^{(j)}))\big) \,\|\, p(r(c(x^{(j)}))) \right)
\end{align*}
\end{tcolorbox}
\begin{proof}
Our primary objective is to maximize $\log p_{\theta}(y \mid c(x))$, i.e. prediction of targetness from the given set of concepts corresponds to a specific region $x$.

We can marginalize the marginal log-likelihood as follows:
\[
\log p(y \mid c(x)) = log \int_{r(c(x))} p(y, r(c(x)) \mid c(x)) \>\>\>dr(c(x))
\]

By introducing a variational distribution $q_{\zeta}(r(c(x))$, we can re-write the above expression as:
\[
\log p_{\theta}(y \mid c(x)) = \log \int_{r\big(c(x)\big)} q_{\zeta}(r(c(x)\mid c(x)\big) \>\> \frac{p_{\phi}(y, r(c(x)) \mid c(x))}{q_{\zeta}(r(c(x)\mid c(x)\big)} \>\>\>dr(c(x)) 
\]
Following the Definition of Expectation, we can write:
\[
\log p_{\theta}(y \mid c(x)) = \log \big[\mathbb{E}_{r(c(x)) \sim q_{\zeta}((r(c(x)\mid c(x)\big))} [\frac{p_{\phi}(y, r(c(x)) \mid c(x))}{q_{\zeta}(r(c(x)\mid c(x)\big)}]\big]
\]

Now, by applying Jensen's Inequality, we can write it as follows: 
\[
\log p_{\theta}(y \mid c(x)) \geq \mathbb{E}_{r(c(x)) \sim q_{\zeta}((r(c(x)\mid c(x)\big))}\bigg[\log\big[\frac{p_{\phi}(y, r(c(x)) \mid c(x))}{q_{\zeta}(r(c(x)\mid c(x)\big)}\big]\bigg]
\]
We can express the above relation as follows:
\begin{align*}
\log p_{\theta}(y \mid c(x)) \geq \mathbb{E}_{r(c(x)) \sim q_{\zeta}((r(c(x)\mid c(x)\big))}\big[p_{\phi}(y, r(c(x)) \mid c(x))\big] \\ -
\mathbb{E}_{r(c(x)) \sim q_{\zeta}((r(c(x)\mid c(x)\big))}\big[q_{\zeta}(r(c(x)\mid c(x)\big)\big]
\end{align*}

We can decompose the joint distribution and rewrite it as follows:
\begin{align*}
\log p_{\theta}(y \mid c(x)) \geq \underbrace{\mathbb{E}_{r(c(x)) \sim q_{\zeta}((r(c(x)\mid c(x)\big))}\big[p(r(c(x)) \mid c(x))\big]}_{\textbf{1}} \\ + \underbrace{\mathbb{E}_{r(c(x)) \sim q_{\zeta}((r(c(x)\mid c(x)\big))}\big[p_{\phi}(y \mid r(c(x)) , c(x))\big]}_{\textbf{2}} \\ 
- \underbrace{\mathbb{E}_{r(c(x)) \sim q_{\zeta}((r(c(x)\mid c(x)\big))}\big[q_{\zeta}(r(c(x)\mid c(x)\big)\big]}_{\textbf{3}}
\end{align*}

By combining term \textbf{1} and \textbf{3}, we can finally express the objective function as:
\begin{align*}
\log p_{\theta}(y \mid c(x)) \geq 
 \mathbb{E}_{r(c(x)) \sim q_{\zeta}((r(c(x)\mid c(x)\big))}\big[p_{\phi}(y \mid r(c(x)) , c(x))\big] \\
- KL( q_{\zeta}(r(c(x)\mid c(x)\big)\big || p(r(c(x)) \mid c(x))\big)
\end{align*}

\end{proof}

\section{Search Performance Across Varying Search Budget}
In this section, we present the search performance of our proposed approach under different search budgets, summarized in Table~\ref{tab:budget_c}. Our results show that increasing the search budget consistently improves performance. As the search budget increases, additional ground truth data becomes available, enabling more effective model parameter optimization. This improved optimization enhances predictive accuracy, which in turn strengthens the overall search performance. 

\begin{table}[!h]
    \centering
    \caption{\textbf{Search performance across varying search budgets $\mathcal{C}$.}}
    \label{tab:budget_c}
    \begin{tabular}{c c c c}
        \toprule
        Search Budget ($\mathcal{C}$) & Accuracy (\%) & F-score (\%) & SR (\%) \\
        \midrule
        10 & 80 $\pm$ 1.5 & 50 $\pm$ 1.1 & 43 $\pm$ 1.3 \\
        30 & 81 $\pm$ 1.3 & 56 $\pm$ 1.2 & 58 $\pm$ 1.6 \\
        50 & \textbf{88 $\pm$ 2.0} & \textbf{65 $\pm$ 2.0} & \textbf{98 $\pm$ 1.0} \\
        \bottomrule
    \end{tabular}
\end{table}

\section{Proof of Proposition 4.2}
\begin{proposition}\label{prop:ca}
The Relevance-Aware Meta-batch formation strategy $\mathcal{D}_t$ facilitates a faster convergence:
$$\mathbb{E}[\mathcal{L}(\theta_T)] - \mathcal{L}(\theta^*) \leq \mathcal{O}\left( \frac{\sigma_{rel}^2}{T} + \frac{D^2}{T^2} \right)$$
Here, the distance from the initial parameters $\theta_1$ to the optimal parameters $\theta^*$ is bounded by $D$: $\left\| \theta_1 - \theta^* \right\|^2 \le D^2$, $\sigma_{rel}^2$ is the variance of loss gradient, i.e. $||\nabla \mathcal{L}_{\mathcal{D}_T}(\theta_T) - \nabla \mathcal{L}(\theta^*)||$ , under the relevance-aware strategy (Detailed in Appendix).
 Here $\theta^*$ is the optimal parameter in hindsight: $\theta^* = \arg\min_\theta \sum_{t=1}^T \mathcal{L}_t(\theta)$.

\end{proposition}

\begin{proof}

We derive the results stated in Proposition~\ref{prop:ca} in two steps.
In step 1, we show that the variance of loss gradient $\sigma_{rel}^2$ from our Relevance-Aware Meta-Batch strategy directly translates to a tighter upper bound on the suboptimality gap $\mathbb{E}[\mathcal{L}(\theta_T)] - \mathcal{L}(\theta^*)]$. In step 2, we formally show that the proposed relevance-aware meta-batch strategy results in  $\sigma_{rel}^2 \leq \mathcal{O} (\frac{1}{N})$, where $N$ is the meta batch size.

\paragraph{Step 1:}
Let $\mathcal{L}(\theta^*)$ be the true expected meta-loss over the stationary task distribution. We assume that the meta-loss $\mathcal{L}$ is $\mu$-strongly convex. Which in turn implies that for any $\theta_1, \theta_2$:$$\mathcal{L}(\theta_1) - \mathcal{L}(\theta_2) \leq \langle \nabla \mathcal{L}(\theta_1), \theta_1 - \theta_2 \rangle - \frac{\mu}{2} \|\theta_1 - \theta_2\|^2$$

Without the loss of generality, we also assume that the gradient estimate $\hat{k}_t = \nabla \mathcal{L}_{\mathcal{D}_t}(\theta_t)$ obtained from our meta-batch $\mathcal{D}_t$ is an unbiased estimator of the true gradient, bounded by the relevance-aware variance. Specifically, we can write this as follows:
\begin{equation}\label{eq:ug}
\mathbb{E}_t[\hat{k}_t] = \nabla \mathcal{L}(\theta_t)
\end{equation}
\begin{equation}
\mathbb{E}_t[\|\hat{k}_t - \nabla \mathcal{L}(\theta_t)\|^2] \leq \sigma_{rel}^2
\end{equation}

We also assume that the distance from the initial parameters $\theta_1$ to the optimal parameters $\theta^*$ is bounded by $D$:$$\|\theta_1 - \theta^*\|^2 \leq D^2$$.

We start with the standard online update rule $\theta_{t+1} = \theta_t - \eta_t \hat{k}_t$, where $\eta_t$ is the learning rate. We measure the squared Euclidean distance to the optimum $\theta^*$:$$\|\theta_{t+1} - \theta^*\|^2 = \|\theta_t - \eta_t \hat{k}_t - \theta^*\|^2$$
Expanding the square yields:$$\|\theta_{t+1} - \theta^*\|^2 = \|\theta_t - \theta^*\|^2 - 2\eta_t \langle \hat{k}_t, \theta_t - \theta^* \rangle + \eta_t^2 \|\hat{k}_t\|^2$$

By taking the expectation of both sides conditioned on the state at step $t$, and leveraging the unbiased gradient estimation property (as in~\ref{eq:ug}), we obtain:
\begin{equation}\label{eq:ip}
\mathbb{E}_t[\|\theta_{t+1} - \theta^*\|^2] = \|\theta_t - \theta^*\|^2 - 2\eta_t \langle \nabla \mathcal{L}(\theta_t), \theta_t - \theta^* \rangle + \eta_t^2 \mathbb{E}_t[\|\hat{k}_t\|^2]
\end{equation}
By rearranging the definition of strong convexity, we can bound the inner product:
\begin{equation}
-\langle \nabla \mathcal{L}(\theta_t), \theta_t - \theta^* \rangle \leq - (\mathcal{L}(\theta_t) - \mathcal{L}(\theta^*)) - \frac{\mu}{2} \|\theta_t - \theta^*\|^2
\end{equation}
Substituting this back into the distance expansion from~\ref{eq:ip}, we obtain: 
\begin{equation}\label{eq:s2}
\mathbb{E}_t[\|\theta_{t+1} - \theta^*\|^2] \leq \|\theta_t - \theta^*\|^2 - 2\eta_t \left( \mathcal{L}(\theta_t) - \mathcal{L}(\theta^*) + \frac{\mu}{2} \|\theta_t - \theta^*\|^2 \right) + \eta_t^2 \mathbb{E}_t[\|\hat{k}_t\|^2]
\end{equation}

The expected squared norm of the stochastic gradient can be decomposed into the squared norm of the true gradient plus the variance, as specified below:$$\mathbb{E}_t[\|\hat{k}_t\|^2] = \|\nabla \mathcal{L}(\theta_t)\|^2 + \mathbb{E}_t[\|\hat{k}_t - \nabla \mathcal{L}(\theta_t)\|^2] \leq \|\nabla \mathcal{L}(\theta_t)\|^2 + \sigma_{rel}^2$$
For simplicity in bounding, in strongly convex optimization with $L$-smoothness, the gradient norm is bounded by the distance to the optimum, which gets absorbed into the convergence constants. We simplify by grouping the variance term:$$\mathbb{E}_t[\|\hat{k}_t\|^2] \leq G^2 \approx \text{const} \cdot \|\theta_t - \theta^*\|^2 + \sigma_{rel}^2$$

Rearranging the inequality from~\ref{eq:s2} to solve for the loss gap $\mathcal{L}(\theta_t) - \mathcal{L}(\theta^*)$:$$2\eta_t (\mathcal{L}(\theta_t) - \mathcal{L}(\theta^*)) \leq (1 - \eta_t \mu) \|\theta_t - \theta^*\|^2 - \mathbb{E}_t[\|\theta_{t+1} - \theta^*\|^2] + \eta_t^2 \sigma_{rel}^2$$
By dividing both sides by $2\eta_t$, we get:
$$\mathcal{L}(\theta_t) - \mathcal{L}(\theta^*) \leq \frac{1 - \eta_t \mu}{2\eta_t} \|\theta_t - \theta^*\|^2 - \frac{1}{2\eta_t} \mathbb{E}_t[\|\theta_{t+1} - \theta^*\|^2] + \frac{\eta_t}{2} \sigma_{rel}^2$$
To ensure convergence, we choose a decreasing learning rate schedule: $\eta_t = \frac{2}{\mu(t+1)}$.Taking the full expectation over all randomness and substituting $\eta_t$:$$\mathbb{E}[\mathcal{L}(\theta_t) - \mathcal{L}(\theta^*)] \leq \frac{\mu(t-1)}{4} \mathbb{E}[\|\theta_t - \theta^*\|^2] - \frac{\mu(t+1)}{4} \mathbb{E}[\|\theta_{t+1} - \theta^*\|^2] + \frac{1}{\mu(t+1)} \sigma_{rel}^2$$
When we sum this inequality from $t=1$ to $T$, the distance terms telescope (cancel each other out), leaving only the initial distance boundary $D^2$ and the summation of the variance:
$$\sum_{t=1}^T \mathbb{E}[\mathcal{L}(\theta_t) - \mathcal{L}(\theta^*)] \leq \frac{\mu \cdot 0}{4} D^2 - \frac{\mu(T+1)}{4} \mathbb{E}[\|\theta_{T+1} - \theta^*\|^2] + \sum_{t=1}^T \frac{1}{\mu(t+1)} \sigma_{rel}^2$$
Because $\sum_{t=1}^T \frac{1}{t+1} \approx \ln(T)$, taking a weighted average of the iterates (Polyak-Ruppert averaging) $\bar{\theta}_T = \frac{2}{T(T+1)}\sum_{t=1}^T t \theta_t$ smooths the trajectory and yields the final deterministic bound:
$$\mathbb{E}[\mathcal{L}(\bar{\theta}_T)] - \mathcal{L}(\theta^*) \leq \frac{2 \sigma_{rel}^2}{\mu T} + \frac{\mu D^2}{T^2}$$
In Big-O notation, dropping the strong convexity constant $\mu$, we arrive exactly at:$$\mathbb{E}[\mathcal{L}(\theta_T)] - \mathcal{L}(\theta^*) \leq \mathcal{O}\left( \frac{\sigma_{rel}^2}{T} + \frac{D^2}{T^2} \right)$$

\paragraph{Step 2:} 
Let $\mathcal{D}_t$ be a meta-batch selected via the proposed relevance-aware meta-batch selection strategy. If the task space $\mathcal{P}$ has an intrinsic dimensionality $d \ll N$, the approximation gradient of loss variance $\sigma_{rel}$ satisfies:$$\sigma_{rel} \leq C \cdot \frac{\text{diam}(\mathcal{P})}{|\mathcal{D}_t|}; \text{where} \>\> \sigma_{rel} = || \hat{k}_{\mathcal{D}_t} - k || = ||\nabla \mathcal{L}_{\mathcal{D}_t}(\theta_t) - \nabla \mathcal{L}(\theta^*)||$$
where $C$ is a constant depending on the smoothness of the meta-loss. This $1/n$ convergence is achieved by minimizing the Fisher Information gap between the batch and the population, effectively aligning the meta-step with the global geometry of the task distribution. Note that, in the case of random meta-batch selection ($\mathcal{D}^{rand}_t)$, the approximate gradient of loss variance $\sigma_{rand}$ is of the order of $\frac{1}{\sqrt{N}}$ (Where {$|\mathcal{D}_t|$= $|\mathcal{D}^{rand}_t|$ = $N$} is the meta-batch size).

Let $\mathcal{P}$ be the latent task space with intrinsic dimensionality $d$. We endow this space with a metric structure such that the maximum distance between any two tasks is bounded by its diameter: $\sup_{T_i, T_j \in \mathcal{P}} \|T_i - T_j\| \leq \text{diam}(\mathcal{P})$.
    
Let $k(T) = \nabla_\theta \mathcal{L}(T; \theta)$ denote the meta-gradient for a specific task $T$.
    
 The gradient mapping $T \mapsto k(T)$ is $L_T$-Lipschitz continuous with respect to the task geometry. For any two tasks $T_i, T_j$:$$\|k(T_i) - k(T_j)\| \leq L_T \|T_i - T_j\|$$The true population meta-gradient is the expectation over the task distribution: $\bar{k} = \mathbb{E}_{T \sim \mathcal{P}}[k(T)]$. Note that $\bar{k}$ lies strictly within the convex hull of the gradient image space $\mathcal{C} = \text{Conv}(\{k(T) : T \in \mathcal{P}\})$.
    
Bound the approximation error $\epsilon = \left\| \hat{k}_{\mathcal{D}_t} - \bar{k} \right\|$, where $\hat{k}_{\mathcal{D}_t} = \sum_{i=1}^n w_i k(T_i)$ is the weighted empirical gradient from the meta-batch $\mathcal{D}_t$ of size $n = |\mathcal{D}_t|$.
    
Step 1: Bounding the Diameter of the Gradient Space: First, we establish the maximum geometric spread of the gradients in the Hilbert space. Due to the $L_T$-Lipschitz smoothness assumption, the diameter of the gradient convex hull $\mathcal{C}$ is directly bounded by the diameter of the latent task space $\mathcal{P}$:$$D_k = \sup_{T_i, T_j \in \mathcal{P}} \|k(T_i) - k(T_j)\| \leq L_T \cdot \sup_{T_i, T_j \in \mathcal{P}} \|T_i - T_j\| = L_T \cdot \text{diam}(\mathcal{P})$$
    
Step 2: Relevance-Aware Selection as Greedy Approximation: 

However, our relevance-aware strategy actively selects a sequence of tasks $T_1, T_2, \dots, T_n$ to maximize coverage and representativeness.

Mathematically, selecting tasks that maximize relevance to the global distribution is equivalent to finding a sparse subset of points in $\mathcal{C}$ to approximate $\bar{k}$. This can be modeled as the Frank-Wolfe (Conditional Gradient) algorithm running on the objective function $J(x) = \|x - \bar{k}\|^2$.
    
Step 3: Iterative Error Reduction: Let $\hat{k}_a$ be the approximation of $\bar{k}$ using a batch of size $a$. The relevance-aware strategy selects the next task $T_{k+1}$ to minimize the residual error. Using the update rule $\hat{k}_{a+1} = (1 - \gamma_a)\hat{k}_a + \gamma_a k(T_{a+1})$ with step size $\gamma_a = \frac{2}{a+2}$, we track the squared error:$$\|\hat{k}_{a+1} - \bar{k}\|^2 = \|(1 - \gamma_a)(\hat{k}_a - \bar{k}) + \gamma_a (k(T_{a+1}) - \bar{k})\|^2$$

Expanding this norm:$$\|\hat{k}_{a+1} - \bar{k}\|^2 = (1 - \gamma_a)^2 \|\hat{k}_a - \bar{k}\|^2 + 2\gamma_a(1 - \gamma_a)\langle \hat{k}_a - \bar{k}, k(T_{a+1}) - \bar{k} \rangle + \gamma_a^2 \|k(T_{a+1}) - \bar{k}\|^2$$

Step 4: Because $T_{a+1}$ is selected via "relevance-aware coverage" (effectively identifying the vertex of the convex hull that aligns most directly with the current residual), it satisfies the linear minimization oracle property:$$\langle \hat{k}_a - \bar{k}, k(T_{a+1}) - \bar{k} \rangle \leq 0$$
Furthermore, the term $\|k(T_{a+1}) - \bar{k}\|^2$ is strictly bounded by the squared diameter of the gradient space $D_k^2$. Substituting these into the expansion from Step 3:$$\|\hat{k}_{a+1} - \bar{k}\|^2 \leq (1 - \gamma_a)^2 \|\hat{k}_a - \bar{k}\|^2 + \gamma_a^2 D_k^2$$

Step 5: We solve this recurrence relation by induction for $a = n$. Base case: For $a=1$, $\|\hat{k}_1 - \bar{k}\|^2 \leq D_k^2$. Assume the induction hypothesis holds for step $a$: $\|\hat{k}_a - \bar{k}\|^2 \leq \frac{4 D_k^2}{a+2}$.
For step $a+1$, plugging in $\gamma_a = \frac{2}{a+2}$:$$\|\hat{k}_{a+1} - \bar{k}\|^2 \leq \left(1 - \frac{2}{a+2}\right)^2 \frac{4 D_k^2}{a+2} + \left(\frac{2}{a+2}\right)^2 D_k^2$$
$$\|\hat{k}_{a+1} - \bar{k}\|^2 \leq \frac{a^2}{(a+2)^2} \frac{4 D_k^2}{a+2} + \frac{4 D_k^2}{(a+2)^2}$$
$$= \frac{4 D_k^2}{(a+2)^2} \left( \frac{a^2}{a+2} + 1 \right) \leq \frac{4 D_k^2}{a+3}$$

Thus, by induction, after forming a meta-batch of size $n = |\mathcal{D}_t|$, the squared approximation error is bounded by:$$\|\hat{k}_{\mathcal{D}_t} - \bar{k}\|^2 \leq \frac{4 D_k^2}{|\mathcal{D}_t| + 2}$$

Step 6: Taking the square root of both sides, the approximation error $\epsilon$ is bounded by:$$\sigma_{rel} = \epsilon = \|\hat{k}_{\mathcal{D}_t} - \bar{k}\| \leq \frac{2 D_k}{\sqrt{|\mathcal{D}_t|}}$$ 
This concludes step 2 of the proof. 
    

\end{proof}

\section{On the Robustness of the Concept Encoder}
A practical challenge in real-world environmental monitoring is that the set of available input channels may vary across datasets or time periods. In our PFAS setting, the 2022 data includes a more comprehensive set of domain-specific channels compared to 2019, reflecting improved data availability over time.

The concept encoder is designed to be robust to such variation by learning representations over heterogeneous multi-channel inputs. Since the encoder operates on the full set of available channels for each dataset, it can flexibly incorporate additional features when present, while still producing meaningful representations when fewer channels are available. 

Empirically, we observe that the model maintains strong performance across both 2019 and 2022 settings, indicating that the learned concept representations generalize despite differences in channel availability.


\section{Influence of Core and Reservoir Buffer Sizes on Search Performance}
In this section, we analyze the impact of the core and reservoir buffer sizes on the search performance. We present the results in Table~\ref{tab:buffer_size}. Results in Table~\ref{tab:buffer_size} show that performance drops when the core buffer dominates the reservoir, while balanced or reservoir-heavy configurations achieve stable and consistently high performance across all metrics. This indicates robustness to buffer size selection provided adequate reservoir capacity is preserved.


\begin{table}[!h]
    \centering
    \caption{\textbf{Search performance across varying buffer configurations.}}
    \label{tab:buffer_size}
    \begin{tabular}{l c c c}
        \toprule
        Buffer configuration & Accuracy (\%) & F1-score (\%) & SR (\%) \\
        \midrule
        Core $>$ Reservoir
            & 77 $\pm$ 1.9 & 50 $\pm$ 1.1 & 88 $\pm$ 0.5 \\
        Core = Reservoir
            & 87 $\pm$ 2.2 & 56 $\pm$ 1.0 & 98 $\pm$ 1.1 \\
        Core $<$ Reservoir (Ours)
            & \textbf{88 $\pm$ 2.0} & \textbf{65 $\pm$ 2.0} & \textbf{98 $\pm$ 1.0} \\
        \bottomrule
    \end{tabular}
\end{table}

\section{Theoretical Justification of the Training Sampling Objective}

\begin{tcolorbox}[colback=gray!10, colframe=gray!40!black, boxrule=0.5pt, arc=1mm]
\begin{theorem}\label{eq:argmax_error}
Assuming \( k \) $= \{X_{train} \setminus \{ x^{(q_1)}, \ldots, x^{(q_{t-1})} \}\}$ represents the set of un-queried datapoints at search step $t$, and that all previously observed regions have equal importance (\( w_i = w_j, \, \forall i, j \)), then the datapoint with the highest entropy $x^{(q_t)}$ can be computed by:
\[
\small{\arg\max_{x^{(q_t)} \in \{X_{train} \setminus \{ x^{(q_1)}, \ldots, x^{(q_{t-1})} \}\}} \left[ \log \sum_{i=0}^{(t-1)} \exp \left( \frac{ \sum_{x^{(q_t)} \in k}^{ }(\mu_{x^{(q_t)}} - \mu_{x^{(q_i)}})^2 }{2\sigma_{x^{(q_t)}}{^2}} \right) \right],}
\]
$\mu_{x^{(q_i)}}$ and $\sigma_{x^{(q_i)}}$ represent the mean and variance of the datapoint queried at $i$-th step. 
\end{theorem}
\end{tcolorbox}
\begin{proof}
In this section, we present the theoretical justification behind our proposed relevance-guided exploration objective as defined in Equation 6. To start with, we view each point in the relevance space as a sample from a standard Gaussian distribution with mean $\mu_{x^{(q_t)}}$ and variance $\sigma_{x^{(q_t)}}$ corresponding to the sample $x^{(q_t)}$. Without loss of generality, we can express it as follows:
\[
p(x^{(q_t)}) =  \mathcal{N}(x; x^{(q_t)}, \sigma_{x^{(q_t)}}^2 I)
\]

Let the Gaussian Mixture Model with all the observed samples and the unobserved sample $x^{(q_t)}$ can be defined as:
\[
p(y) = \sum_{i=1}^{(t-1)} w_i\, \mathcal{N}(y; \mu_{x^{(q_i)}}, \sigma_{x^{(q_i)}}^2 I)
\]
where each sample has mean $\mu_{x^{(q_i)}}$ and covariance $\sigma_{x^{(q_i)}}^2 I$. As in our setting, all the previously observed samples are equally important, so we assign $w_i = 1$ for all the samples. 

We seek its (differential) entropy. Following the definition of entropy $H$, we can express it as follows:
\[
H[p] = -\int p(y) \log p(y) \, dy
\]

\text{Plugging it into the Gaussian Mixture Model, we obtain}

\[
H[p] = -\int \left( \sum_{i=1}^{(t-1)} w_i\, \mathcal{N}(y; \mu_{x^{(q_i)}}, \sigma_{x^{(q_i)}}^2 I) \right) \log \left( \sum_{j=1}^{(t-1)} w_j\, \mathcal{N}(y; \mu_{x^{(q_j)}}, \sigma_{x^{(q_j)}}^2 I) \right) dy
\]
Rewriting the order of summation and integration, we obtain:
\[
H[p] = -\sum_{i=1}^{(t-1)} w_i \int \mathcal{N}(y; \mu_{x^{(q_i)}}, \sigma_{x^{(q_i)}}^2 I) \log\left( \sum_{j=1}^{N_p} w_j\, \mathcal{N}(y; \mu_{x^{(q_j)}}, \sigma_{x^{(q_j)}}^2 I) \right) dy
\]

\text{Next, we utilize the rule of Gaussian Product Kernel Evaluation.}

For two isotropic Gaussians, we can express it as follows:
\[
\mathcal{N}(y; \mu_i, \sigma^2 I)\mathcal{N}(y; \mu_j, \sigma^2 I) \propto \mathcal{N}(\mu_i; \mu_j, 2\sigma^2 I)
\]
Note that, in the above formulation, we assume that the variances are the same for two different samples in the relevance space.  
So the kernel integral becomes:
\begin{equation}\label{eq:ho}
\int \mathcal{N}(y; \mu_i, \sigma^2 I)\mathcal{N}(y; \mu_j, \sigma^2 I)\,dy = (2\pi\sigma^2)^{-d/2} \exp\left(-\frac{1}{4\sigma^2}\|\mu_i - \mu_j\|^2\right)
\end{equation}

\text{Utilizing the above relation~\ref{eq:ho} and according to the (Hershey-Olsen Bound)~\citep{hershey2007approximating},} approximating the expectation by centering at each mean ($x \leftarrow \mu_{x^{(q_i)}}$), we get:

\[
H[p] \approx \text{const} + \sum_{i=1}^{(t-1)} w_i \log\left( \sum_{j=1}^{(t-1)} w_j\, \exp\left( \frac{\|\mu_{x^{(q_t)}}-\mu_{x^{(q_j)}}\|^2}{2\sigma_{x^{(q_t)}}^2}\right)\right)
\]

Hence, the entropy is proportional to:
\begin{equation}\label{eq:ent}
H[p] \propto \sum_{i=1}^{(t-1)} w_i \log\left( \sum_{j=1}^{(t-1)} w_j\, \exp\left( \frac{\|\mu_{x^{(q_t)}}-\mu_{x^{(q_j)}}\|^2}{2\sigma_{x^{(q_t)}}^2}\right)\right)
\end{equation}

As we assume each observed sample is equally important, thus, $w_i = w_j=1$.

Setting $w_i = w_j=1$, we can write the above expression as:
\[
H[p] \propto \log\left( \sum_{j=1}^{(t-1)} \, \exp\left( \frac{\|\mu_{x^{(q_t)}}-\mu_{x^{(q_j)}}\|^2}{2\sigma_{x^{(q_t)}}^2}\right)\right)
\]

Hence, utilizing the above expression, we can query an unobserved sample $x^{(q_t)}$ that corresponds to maximum entropy. We can choose $x^{(q_t)}$ based on the following criteria:

\[
\small{\arg\max_{x^{(q_t)} \in \{X_{train} \setminus \{ x^{(q_1)}, \ldots, x^{(q_{t-1})} \}\}} \left[ \log \sum_{i=0}^{(t-1)} \exp \left( \frac{ \sum_{x^{(q_t)} \in k}^{ }(\mu_{x^{(q_t)}} - \mu_{x^{(q_i)}})^2 }{2\sigma_{x^{(q_t)}}{^2}} \right) \right],}
\]

\end{proof}

  

\section{Theoretical Justification of the Inference Sampling Objective}
\begin{tcolorbox}[colback=gray!10, colframe=gray!40!black, boxrule=0.2pt, arc=1mm]
\begin{theorem}\label{eq:like-ecore}
Assuming policy parameter, $\theta^{t-1}$, the expected log-likelihood of unqueried data $x^{(q_t)}$ containing target $y^{(q_t)}$ can be expressed:
\[
 \exp \left\{ - \frac{\sum_{i=0}^{(t-1) }(\mu_{x^{(q_t)}} - \mu_{x^{(q_i)}})^2}{2\sigma_{x^{(q_t)}}^2} \right\}  \big(\sum_{i=0}^{P} \exp^{\pi_{\theta^{t-1}}(x^{(q_t)}_i)}\big)
\]
\end{theorem}
\end{tcolorbox}

\begin{proof}
The expected log-likelihood of unqueried data $x^{(q_t)}$ is composed of two terms. Let's first analyze the first term. 
We start with the definition of entropy $H$:
\[
\mathbb{E}_{x^{(q_t)}}[\log p(x^{(q_t)} | \tilde{x}_{t-1})]=-H(x^{(q_t)} | \tilde{x}_{t-1})
\]
Here, $\tilde{x}_{t-1}$ represents the set of already observed datapoints, i.e., $\tilde{x}_{t-1} = \{X_{train} \setminus \{ x^{(q_1)}, \ldots, x^{(q_{t-1})} \}\}$.
Substituting the expression of $H(x^{(q_t)} | \tilde{x}_{t-1})$ as defined in Equation~\ref{eq:ent}, and by setting \( w_i = w_j = 1 \), we obtain:
\[
\mathbb{E}_{x^{(q_t)}}[\log p(x^{(q_t)} | \tilde{x}_{t-1})] \propto -  \big( \sum^{(t-1)}_{j=1} \exp \left\{ {\frac{||\mu_{x^{(q_t)}} - \mu_{x^{(q_j)}}||^2_2}{2\sigma^2_{x^{(q_t)}}}} \right\} \big)
\]


By simplifying, we obtain,
\[
\mathbb{E}_{x^{(q_t)}}[\log p(x^{(q_t)} | \tilde{x}_{t-1})] \propto -  \sum^{(t-1)}_{j=1} \left\{ {\frac{||\mu_{x^{(q_t)}} - \mu_{x^{(q_j)}}||^2_2}{2\sigma^2_{x^{(q_t)}}}} \right\} 
\]

By leveraging the monotonicity property of the exponential function, we can rewrite the above expression as:

\begin{equation}\label{eq:rel_enc_pr}
\mathbb{E}_{x^{(q_t)}}[\log p(x^{(q_t)} | \tilde{x}_{t-1})] \propto \exp \left\{ -   {\frac{\sum^{(t-1)}_{j=1} ||\mu_{x^{(q_t)}} - \mu_{x^{(q_j)}}||^2_2}{2\sigma^2_{x^{(q_t)}}}} \right\} 
\end{equation}
The above expression resembles the first term of the expected log-likelihood as stated in the proposed theorem. It is important to note that our decision-making policy is composed of two parts: (a) the relevance encoder that computes $\mu_{x^{(q_t)}}$ and $\sigma^2_{x^{(q_t)}}$; (b) the relevance decoder that computes the likelihood of containing target $y^{(q_t)}$ given $\mu_{x^{(q_t)}}$ and $\sigma^2_{x^{(q_t)}}$. Which can be represented as follows:

\begin{equation}\label{eq:pred_cond}
\mathbb{E}_{x^{(q_t)}}[\log p(y^{(q_t)} | x^{(q_t)}, \tilde{x}_{t-1})] \propto \sum_{i=0}^{P} \pi_{\theta^{t-1}}(x^{(q_t)}_i)
\end{equation}

Note, $i$ denotes the pixel index, and $P$ is the number of pixels in $x^{(q_t)}$. Our final objective is to compute the expected log-likelihood of a given geospatial region $x^{(q_t)}$ containing the target $y^{(q_t)}$, which can be expressed as:

Which we can rewrite as follows:
\[
\mathbb{E}_{x^{(q_t)}}[\log p(x^{(q_t)} | \tilde{x}_{t-1}) \cdot \log p(y^{(q_t)} | x^{(q_t)}, \tilde{x}_{t-1})] 
\]

By utilizing the expression of ~\ref{eq:rel_enc_pr} and~\ref{eq:pred_cond}, we can express the above quantity as follows:
\begin{align*}
\underbrace{\mathbb{E}_{x^{(q_t)}}[\log p(x^{(q_t)} | \tilde{x}_{t-1}) \cdot \log p(y^{(q_t)} | x^{(q_t)}, \tilde{x}_{t-1})]}_{\textit{Expected Log-likelihood of $x^{(q_t)}$ containing the target $y^{(q_t)}$}} \propto  
\exp \left\{ -   {\frac{\sum^{(t-1)}_{j=1} ||\mu_{x^{(q_t)}} - \mu_{x^{(q_j)}}||^2_2}{2\sigma^2_{x^{(q_t)}}}} \right\} \\ 
\cdot \sum_{i=0}^{P} \pi_{\theta^{t-1}}(x^{(q_t)}_i)
\end{align*}

By utilizing the monotonicity property of the exponential function, we can rewrite the above expression as:

\begin{align*}
\underbrace{\mathbb{E}_{x^{(q_t)}}[\log p(x^{(q_t)} | \tilde{x}_{t-1}) \cdot \log p(y^{(q_t)} | x^{(q_t)}, \tilde{x}_{t-1})]}_{\textit{Expected Log-likelihood of $x^{(q_t)}$ containing the target $y^{(q_t)}$}} \propto 
\exp \left\{ -   {\frac{\sum^{(t-1)}_{j=1} ||\mu_{x^{(q_t)}} - \mu_{x^{(q_j)}}||^2_2}{2\sigma^2_{x^{(q_t)}}}} \right\} \\ 
\cdot \sum_{i=0}^{P} \exp^{\pi_{\theta^{t-1}}(x^{(q_t)}_i)}
\end{align*}

\end{proof}

\section{Details of Greedy Intersection Clustering Algorithm}
This section details the greedy-intersection clustering algorithm used to cluster the relevance vectors in the core buffer, as highlighted in the meta-training set formation subsection in the main paper. The Greedy Intersection clustering algorithm was originally proposed in~\cite{ge2025learning}. We present it in detail for completeness. 
Consider a set $\theta \in T$ of relevance vectors, fix $K$ (i.e., number of clusters), and suppose $\epsilon$ is the radius of each cluster.
The key intuition behind the Greedy Intersection Algorithm is that for any $\theta \in T$, an $\epsilon$-hypercube centered at $\theta$ characterizes all possible $\theta'$ that can cover $\theta$ in the sense that $|| \theta - \theta'||_{\infty} \leq \epsilon$. Consequently, if two $\epsilon$-hypercubes centered at $\theta$ and $\theta'$ overlap, any point within their intersection can simultaneously cover both. We demonstrate this with the following straightforward example:

\begin{center}
    \begin{tikzpicture} [scale=2]

\def\cross{%
    \draw (-0.05,-0.05) -- (0.05,0.05);
    \draw (-0.05,0.05) -- (0.05,-0.05);
}

\draw[thick] (0,0) coordinate (A) -- (0.7,0) coordinate (B);
\draw (A) node[left] {$[x_1] \hspace{1.3cm} x_1-\epsilon$};
\draw (B) node[right] {$x_1+\epsilon$};

{%
    \draw (0.35-0.05,-0.05) -- (0.35+0.05,0.05);
    \draw (0.35-0.05,0.05) -- (0.35+0.05,-0.05);
}

\draw[thick] (0.2,0.2) coordinate (C)-- (0.9,0.2)coordinate (D);
\draw (C) node[left] {$[x_1,x_2] \hspace{1.2cm} x_2-\epsilon$};
\draw (D) node[right] {$x_2+\epsilon$};

{%
    \draw (0.55-0.05,0.2-0.05) -- (0.55+0.05,0.2+0.05);
    \draw (0.55-0.05,0.2+0.05) -- (0.55+0.05,0.2-0.05);
}

\draw[thick] (0.5,0.4) coordinate (E) -- (1.2,0.4) coordinate (F);
\draw (E) node[left] {$[x_1,x_2,x_3] \hspace{1.3cm} x_3-\epsilon$};
\draw (F) node[right] {$x_3+\epsilon$};
{%
    \draw (0.85-0.05,0.4-0.05) -- (0.85+0.05,0.4+0.05);
    \draw (0.85-0.05,0.4+0.05) -- (0.85+0.05,0.4-0.05);
}

\draw[thick] (1.1,0.6) coordinate (I) -- (1.8,0.6) coordinate (J);
\draw (I) node[left] {$[x_3,x_4] \hspace{3cm} x_4-\epsilon$};
\draw (J) node[right] {$x_4+\epsilon$};

{%
    \draw (1.45-0.05,0.6-0.05) -- (1.45+0.05,0.6+0.05);
    \draw (1.45-0.05,0.6+0.05) -- (1.45+0.05,0.6-0.05);
}
\draw[dashed,red] (1.1,-0.1) -- (1.1,0.8);
\draw[dashed,red] (1.2,-0.1) -- (1.2,0.8);
\draw[dashed,blue] (0.7,-0.1) -- (0.7,0.8);
\draw[dashed,blue] (0.5,-0.1) -- (0.5,0.8);

\end{tikzpicture}
\end{center}

Each cross marks a target parameter to be covered, and each line segment shows the range within which an $\epsilon$-close representative of that parameter may lie. Selecting a point within the overlapping region of these intervals allows us to simultaneously cover multiple parameters.

The Greedy Intersection algorithm builds on this intuition by generalizing it as follows. In the first stage, it constructs an intersection tree independently for each dimension. For the $s$-th dimension, the data points are sorted in ascending order based on their 
$s$-th coordinate, denoted as $x_1 < x_2 \dots <x_n$. For each point $x_i$, we initialize a list containing only $[x_i]$, which will be used to track how many other points can be jointly covered along with it.

Starting from the second smallest datapoint $x_2$, we check if $x_2-\epsilon \le x_1+\epsilon,$ i.e. if $x_2\le x_1+2\epsilon.$ Since $x_2-\epsilon>x_1-\epsilon$ due to our sorting, any point inside $[x_2-\epsilon,x_1+\epsilon]$ can cover both $x_1,x_2.$ 
Consequently, if this interval is valid, we append $x_1$ to $x_2$'s list to indicate that both points can be covered simultaneously. More generally, for any point $x_i$, we iterate backward through the preceding points $x_j$ (from $j = i-1$ down to $1$) to verify if $x_i \le x_j + 2\epsilon$. If the condition is met, we add $x_j$ to $x_i$'s list. Because the set is sorted, as soon as the condition fails, we can stop checking; if $x_i > x_j + 2\epsilon$, it is guaranteed that $x_i > x_{j'} + 2\epsilon$ for all $j' < j$. Ultimately, the common covering interval for all points in $x_i$'s list is $[x_i - \epsilon, x_j + \epsilon]$, where $j$ represents the smallest index successfully added to the list.
There are $1+2+\dots+n-1= \mathcal{O}(n^2)$ comparisons in total. 
We form a set of these lists, and call it $\mathcal{A}_s$ for the $s$-th dimension. The figure above illustrates how the algorithm works to find out $\mathcal{A}_1= \{[x_1],[x_1,x_2],[x_1,x_2,x_3], [x_3,x_4]\}$. 

The second stage of the algorithm aims to identify a hypercube that covers the maximum number of points, selecting one axis-aligned interval from each dimension. By the geometry of Euclidean space, two points $\theta_1$ and $\theta_2$ are within an $\epsilon$-distance in $\ell_\infty$ norm if and only if they appear together in each other's lists across all dimensions. Thus, to find a hypercube whose center is within $\ell_\infty$-distance of the most data points, we search for the combination of lists $l_1, \dots, l_d$, where each $l_s$ is chosen from the corresponding set $\mathcal{A}_s$, such that their intersection has the maximum cardinality. In our illustrative example, we observe that the groups $[x_1, x_2, x_3]$ and $[x_3, x_4]$ must be covered separately, each by a different point positioned between the red or blue vertical lines.

A key limitation of the Greedy Intersection Algorithm (GIA) is its exponential complexity with respect to the number of dimensions in the data. 
To address the challenge of high dimensionality, we apply standard dimensionality reduction techniques—specifically, Principal Component Analysis (PCA) - to project each data point into a lower-dimensional space that is computationally more tractable. We then apply the Greedy Intersection Algorithm to these PCA-transformed points. In our experiments, we use a reduced dimensionality of 7, which we found to yield strong performance in our problem setting. The complete algorithm is presented below in Algorithm~\ref{alg:greedy_intersection}.

\begin{algorithm}[h]
    \caption{Greedy Intersection}
    \label{alg:greedy_intersection}
    \textbf{Input}: $T = \{\theta_i\}_{i=1}^N$, $\epsilon > 0$, $K \ge 1$ \\
    \textbf{Output}: Parameter cover $C$
    \begin{algorithmic}[1]
    \State $C \gets []$
        \For{$round\ k$ = $1$ to $K$}
            \For{$dimension\ m$ = $1$ to $d$}
                \State Sort $T$ in ascending order based on their $m$-th coordinates
                \State $lists_m \gets []$
                \For{$indiviual\ i = 2$ to $N$}
                    \State $S_i \gets [\theta_i]$
                    \For{$j = i-1$ to $1$}
                        \If{$\theta_i$'s $m$-th coordinate $< \theta_j$'s $m$-th coordinate $+2\epsilon$}
                            \State Add $\theta_j$ to $S_i$
                        \Else
                            \If{ $lists_m[-1] \subseteq S_{i}$ }
                                \State $lists_m[-1] \gets S_{i}$ 
                            \Else
                                \State Add $S_{i}$ to $lists_m$
                            \EndIf
                            \State \textbf{break}
                        \EndIf
                    \EndFor
                \EndFor
            \EndFor

            \State $S^{1*},\dots, S^{m*} \gets  \text{argmax}_{S^1\in lists_1,\dots, S^m\in lists_m} |S^1\cap \dots \cap S^m|$
             \State $covered \gets S^{1*}\cap \dots \cap S^{m*}$
              \State $\hat{\theta}_k \gets $ average of the $covered$
            \State $T \gets T - covered$
            \State $C$.adds($\hat{\theta}_k$)
        \EndFor
        \State \textbf{return} $C$
    \end{algorithmic}
\end{algorithm}
Additionally, to evaluate the effectiveness of the Greedy Intersection clustering algorithm, we conduct a comparative analysis by replacing it with a standard clustering method (i.e., DBScan) and measuring the resulting performance on the downstream task. The outcomes of this comparison are summarized in the following table. Our empirical results indicate that incorporating the Greedy Intersection clustering approach within the framework achieves improved performance on accuracy and SR metrics, while yielding a significant improvement in F-score. This improvement indicates enhanced robustness of the framework, reflecting a better balance between precision and recall under the proposed clustering strategy, thereby underscoring the effectiveness of the Greedy Intersection clustering approach under the proposed framework.

\begin{table}[!h]
    \centering
    \caption{\textbf{Effectiveness of the clustering approach.}}
    \label{tab:ortho_}
    \begin{tabular}{l c c c}
        \toprule
        Clustering approach & Accuracy (\%) & F-score (\%) & SR (\%) ($\mathcal{C}=50$) \\
        \midrule
        DBSCAN & 81 $\pm$ 2.2 & 57 $\pm$ 1.6 & 94 $\pm$ 1.1 \\
        Ours   & \textbf{88 $\pm$ 2.0} & \textbf{65 $\pm$ 2.0} & \textbf{98 $\pm$ 1.0} \\
        \bottomrule
    \end{tabular}
\end{table}

\section{Additional Analysis on the Importance of Relevance Guided Sampling}

Our adoption of a Conditional Variational Autoencoder (CVAE) as the backbone of our relevance estimator is driven by the inherent dynamism of open-world geospatial search environments, where the relevance of concepts can shift unpredictably across both spatial and temporal dimensions. The CVAE framework enables the model not only to learn complex relevance patterns but also to quantify associated uncertainties - empowering the system to make decisions that are both adaptive and robust to changing conditions. This uncertainty-aware modeling is especially crucial in large-scale and data-sparse geospatial scenarios, where assumptions of static or simple relevance easily break down.

In rigorous empirical comparisons, the CVAE-based estimator consistently outperforms a simpler multi-layer perceptron (MLP) attention-based relevance estimation baseline (denoted as Ours (RG $\rightarrow$ MLP)), delivering improvements in performance, as demonstrated in the results in Table~\ref{tab:Rel_RG} and also detailed in Section 5 (Tables~\ref{tab:ablation}). These findings underscore the substantial practical advantage of leveraging probabilistic deep generative modeling for relevance estimation, affirming that the CVAE’s ability to jointly represent relevance values and their uncertainties translates directly to superior real-world search effectiveness. 

\begin{table}[!h]
    \centering
    \caption{\textbf{Analyzing the impact of relevance-guided sampling.}}
    \label{tab:Rel_RG}
    \begin{tabular}{l c c c}
        \toprule
        Method & Accuracy (\%) & F-score (\%) & SR (\%) \\
        \midrule
        Ours (RG $\rightarrow$ MLP)
            & 85 $\pm$ 1.7 & 56 $\pm$ 1.8 & 94 $\pm$ 2.3 \\
        Ours
            & \textbf{88 $\pm$ 2.0} & \textbf{65 $\pm$ 2.0} & \textbf{98 $\pm$ 1.0} \\
        \bottomrule
    \end{tabular}
\end{table}

\section{Importance of Orthogonalization in the Concept Space}
We assess the effectiveness of the orthogonalization layer by comparing the performance of the proposed framework with a variant that only excludes orthogonalization. The results are presented in Table~\ref{tab:ortho_c}. The empirical results indicate a noticeable decline in the metrics when the orthogonalization layer is omitted from the proposed framework, highlighting the importance of the orthogonalization of the concepts within the proposed framework. 
\begin{table}[!h]
    \centering
    \caption{\textbf{Importance of concept orthogonalization.}}
    \label{tab:ortho_c}
    \begin{tabular}{l c c c}
        \toprule
        Orthogonalization & Accuracy (\%) & F-score (\%) & SR (\%) ($\mathcal{C}=50$) \\
        \midrule
        No
            & 80 $\pm$ 1.0 & 56 $\pm$ 1.5 & 86 $\pm$ 0.7 \\
        Yes (our framework)
            & \textbf{88 $\pm$ 2.0} & \textbf{65 $\pm$ 2.0} & \textbf{98 $\pm$ 1.0} \\
        \bottomrule
    \end{tabular}
\end{table}

\section{Dataset and Split Details}
We evaluate our approach on two geospatial prediction tasks: PFAS contamination prediction and land cover classification, both of which use satellite data aligned with sparsely sampled ground truth.

For the PFAS contamination prediction task, each sample consists of a multi-channel raster patch centered on a known PFAS measurement site from EPA datasets, including the National Rivers and Streams Assessment (NRSA)~\footnote{\href{https://www.epa.gov/national-aquatic-resource-surveys/nrsa}{NRSA}} and the National Lakes Assessment Fish Tissue Study~\footnote{\href{https://www.epa.gov/choose-fish-and-shellfish-wisely/2022-national-lakes-assessment-fish-tissue-study}{National Lakes Assessment Fish Tissue Study}}. This dataset is inherently sparse, with only 704 labeled sample points across the US. Each patch is of size 256 × 256 pixels at 30-meter resolution and contains a set of multi-channel geospatial features, with the number of channels varying across datasets (33 for PFAS 2019 and 41 for PFAS 2021). These channels encode diverse geospatial features, such as land cover rasters from the National Land Cover Database (NLCD)~\footnote{\href{https://www.usgs.gov/centers/eros/science/national-land-cover-database}{NLCD}}, flow direction rasters capturing hydrological connectivity obtained from the ArcGis Pro software, and distance transforms to known PFAS discharger sites obtained from the U.S EPA database~\footnote{\href{https://echo.epa.gov/}{US EPA}}. This rich set of environmental and spatial layers provides the model with meaningful context to predict contamination.

For the land cover classification task, we use imagery from the Sentinel-2 surface reflectance dataset~\footnote{\href{https://developers.google.com/earth-engine/datasets/catalog/COPERNICUS_S2_SR_HARMONIZED}{Sentinel-2}}. For each sample region, we extract 4-channel image patches using the Near-Infrared (NIR), Red, Green, and Blue bands. Each patch is of size 256 × 256 pixels at 30-meter resolution. These patches are aligned with the same spatial footprint as the PFAS patches to ensure consistency in coverage. 

All datasets used are publicly available and used in accordance with their respective licenses (e.g., U.S. EPA datasets, NLCD), which permit research use.

\paragraph{Dataset Splits and Query Budgets.}
We construct disjoint training and test splits under both spatial and temporal generalization settings as shown in Figure \ref{fig:dataset_splits}. In the spatial setting, training and test samples are drawn from non-overlapping geographic regions within the same year (2019), while in the temporal setting, models are trained on 2019 data and evaluated on 2022 data to capture distribution shift. For each setting, we ensure that training and test sets contain no overlapping sample locations or spatial patches. During training, we use a fixed query budget of 150 samples to simulate limited labeled data acquisition. At inference, we evaluate performance under a primary query budget of 50 samples, with additional budgets considered in ablations. For the land cover (LC) task, we follow the same protocol but simulate sparse supervision by randomly subsampling labels, treating it as a controlled generalization experiment rather than a real-world acquisition setting.

\begin{figure*}[h]
    \centering

    \begin{subfigure}{0.9\textwidth}
        \centering
        \includegraphics[width=\linewidth]{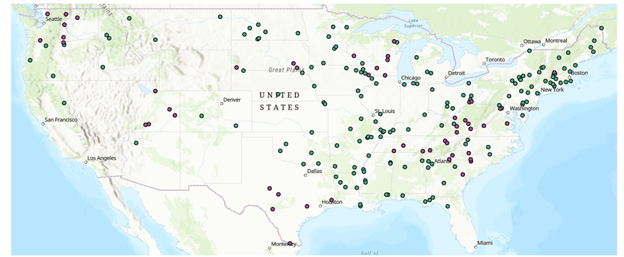}
        \caption{Spatial generalization (2019): training and test samples are drawn from disjoint geographic regions within the same year.}
        \label{fig:spatial_split}
    \end{subfigure}

    \vspace{6pt}

    \begin{subfigure}{0.9\textwidth}
        \centering
        \includegraphics[width=\linewidth]{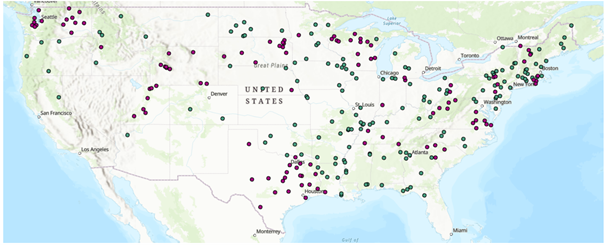}
        \caption{Temporal generalization: training samples from 2019 and test samples from 2021, capturing distribution shift over time.}
        \label{fig:temporal_split}
    \end{subfigure}

    \caption{\textbf{PFAS dataset splits.} Green denotes training samples and purple denotes test samples. Splits are geographically disjoint; apparent overlap arises from visualization at this scale.}
    \label{fig:dataset_splits}
\end{figure*}

\section{Details of Pseudolabel Generation Procedure to Induce Stability During Training}
Given the scarcity of ground truth data in PFAS monitoring, typically limited to sparse point-based measurements, we employ a point-based label expansion strategy to generate dense supervision signals suitable for segmentation training. Specifically, each labeled point is used to generate a pseudo-label mask over its surrounding patch, enabling the model to learn from broader spatial context~\footnote{Citation withheld for anonymity.}. Special \enquote{noise masks} are used in the loss computation to reduce the influence of uncertainty introduced due to the pseudo-labeling assumptions.

For each georeferenced point with a PFAS contamination label (1 for above health advisory threshold, 0 for below), we extract a raster patch centered on that point. All surface water pixels within this patch are assigned the same label as the central point, reflecting the assumption, guided by environmental science, that contamination may be correlated in hydrologically connected regions. Non-surface water pixels are assigned a special label (2) and are excluded from loss computation. This strategy effectively converts sparse point annotations into dense training masks, allowing the model to be trained as a segmentation network despite the absence of full pixel-wise labels. 

Crucially, model evaluation is performed exclusively on ground truth point data from the U.S. EPA PFAS datasets, ensuring that performance metrics reflect real-world observations and are not influenced by the assumptions underlying pseudo-label generation. 

For the land cover dataset, dense per-pixel land cover annotations are used as supervision only during training, and are derived from public land cover products that include standard classes such as water, forest, urban, and agricultural land~\footnote{\href{https://www.usgs.gov/centers/eros/science/national-land-cover-database}{NLCD}}. During testing, however, we adopt a sparse label evaluation setup: only one or two labeled pixels are retained in each test mask, simulating the limited label availability often encountered in practice. The model is expected to make predictions over the entire patch, but evaluation is conducted only at these sparsely labeled locations. This setup mirrors the PFAS task and encourages the model to generalize beyond isolated supervision.

\subsection{Additional Discussions on the Label Cost}
The OWL-GPS setting models constraints on supervision, not imagery. For instance, \textbf{Sentinel imagery is abundant; however, the labels required for the discovery tasks in our datasets are not. In real scientific monitoring problems like PFAS contamination, labeled examples are scarce, expensive, and slow to acquire.} Moreover, OWL-GPS assumes an online, memory-limited regime in which the agent cannot store or replay past inputs beyond a small fixed buffer. These are the core constraints the framework is designed to capture.

\begin{enumerate}
    \item PFAS (2019 $\&$ 2022): Each PFAS label corresponds to an actual field measurement collected by environmental agencies. Such measurements require field deployment, sample collection, laboratory chemical analysis, and administrative coordination. As a result, very few labeled samples exist, and expanding coverage is constrained by budget, personnel, and lab throughput. This is the source of \enquote{costly} and \enquote{resource-constrained.} In addition, under OWL-GPS assumptions, inputs are non-replayable: once a geospatial region has been processed and leaves the limited memory buffer, it cannot be revisited for training or queried again. The agent therefore must decide which regions to label as they arrive.
    \item Land Cover (LC): For LC, labels are not inherently difficult to acquire, but the purpose of including LC is to demonstrate that the framework generalizes beyond PFAS. To place LC into a discovery setting comparable to PFAS, we \textbf{sparsify the supervision as described in the Appendix O}. This creates a regime where only a small subset of tiles are labeled, matching the limited-supervision assumption of OWL-GPS. We clarify that this sparsification is simulated to evaluate generalization, not a claim about real-world LC label difficulty.
\end{enumerate}

\section{Details of Training and Inference Hyperparameters}
We used a patch size of 256×256 pixels for all geospatial regions. Model training was performed using the AdamW optimizer, with learning rate dynamically adjusted using a polynomial decay scheduler with warmup. Training used focal loss with a KL-divergence regularization term as part of a conditional variational autoencoder objective for relevance learning. Online-meta updates were performed using meta-batches of size 10, formed by selecting one representative sample per cluster from the core buffer, and three samples from the reservoir buffer. All reported results are averaged over three independent runs with different random seeds to ensure robustness.

\section{Architecture Details }
\subsection{Details of Concept Encoder}
This section details the implementation of the concept encoder as highlighted in the concept encoder subsection in the main paper.
The Concept Encoder is implemented using a modified Vision Transformer (ViT) architecture, tailored to handle geospatial imagery possibly with temporal structure. The encoder is structured as a deep transformer-based module that ingests high-dimensional multi-spectral or multi-temporal satellite imagery and maps it into a low-dimensional concept space that captures the essential semantics of the input.

The encoder uses a patch-based embedding mechanism to divide the input tensor into non-overlapping spatial patches. Each patch is linearly projected into an embedding vector, to which a learnable class token and fixed sinusoidal 3D positional encodings are added. These embeddings are then processed through a deep stack of transformer blocks, each consisting of multi-head self-attention and MLP layers, followed by a final normalization step.

The encoder is initialized using a sin-cos positional embedding and can optionally load pretrained weights. To improve representation disentanglement and mitigate redundancy, we apply Gram-Schmidt orthogonalization to the learned latent vectors, which promotes diversity among the learned concept axes. The output of the encoder is a tuple of patch-wise embeddings enriched with contextual information, which serve as the base representations for relevance inference and downstream target prediction.

\paragraph{Training Objective.}
The concept encoder is pretrained in a masked autoencoding (MAE) framework following prior geospatial foundation models. Given an input patch $x \in \mathbb{R}^{H \times W \times K}$, a subset of patches is randomly masked, and the encoder processes only the visible tokens. A lightweight decoder then reconstructs the masked regions.

We optimize a reconstruction loss over masked patches:
\[
\mathcal{L}_{\text{recon}} =
\frac{1}{|\mathcal{M}|}
\sum_{(h,w)\in \mathcal{M}} \|x_{h,w} - \hat{x}_{h,w}\|_2^2,
\]
where $\mathcal{M}$ denotes the set of masked patches, $\hat{x}$ represents the reconstructed masked patch. This encourages the latent representation to capture the underlying spatial and semantic structure of the multi-channel geospatial input.

\subsection{Details of Relevance Encoder}
This section details the implementation of the relevance encoder as highlighted in the relevance encoder and decoder subsection in the main paper.
The Relevance Encoder estimates a latent relevance distribution over concepts to capture their varying importance in influencing target presence across geospatial regions. It is modeled as the probabilistic encoder component in a Conditional Variational Autoencoder (CVAE), and takes as input the region-specific concept embeddings from the concept encoder.

This module is designed to process batches of concept embeddings organized per concept, and leverages a lightweight yet expressive convolutional backbone consisting of three convolutional layers with ReLU activations and batch normalization, followed by spatial downsampling via max-pooling. The resulting feature maps are flattened and passed through two separate linear layers to produce the mean and standard deviation vectors, one for each concept. These vectors parameterize a diagonal Gaussian distribution from which a sample relevance vector is drawn during training via the reparameterization trick.

The output relevance vector quantifies the relative contribution and uncertainty of each concept with respect to the target prediction task. This facilitates not only accurate modeling but also reasoning about the influence of different domain factors. Empirically, the relevance encoder was shown to significantly boost predictive accuracy and search efficiency.

\subsection{Details of the Decoder}
This section details the implementation of the decoder as highlighted in the relevance encoder and decoder subsection in the main paper.
The Decoder is responsible for producing pixel-level predictions of target presence by conditioning on the fused representation of concepts and their corresponding relevance scores. It is designed using a dual-head FCN (Fully Convolutional Network) architecture to enhance segmentation robustness during learning. 

Each head is a configurable FCN variant that processes the fused feature map via a sequence of convolutional layers, followed by feature concatenation and a final classification layer. The first head is lightweight with a single convolution, while the second includes a deeper convolutional stack to model more complex interactions.

The decoder architecture aligns with semantic segmentation principles and supports variable numbers of output classes, dropout for regularization, and batch normalization to stabilize training. This component plays a critical role in translating abstract latent representations into actionable geospatial predictions.

\subsection{Modeling Spatial and Hydrological Structure.} 
While OWL-GPS does not explicitly incorporate graph-based message passing over river networks, spatial and hydrological dependencies are encoded through the input representation. Specifically, the model ingests domain-informed channels such as flow direction rasters, distance-to-source features, and land cover, which capture hydrological connectivity and spatial context. As a result, regions that are hydrologically linked (e.g., upstream/downstream) exhibit similar patterns in the concept space, enabling the model to implicitly learn these dependencies. Importantly, OWL-GPS operates in a sequential, budget-constrained setting where the full spatial graph is not known a priori and evolves as new regions are queried. This makes explicit graph construction and propagation impractical, as it would require access to the entire network structure and labeled data across nodes. Instead, our approach leverages learned relevance over concept representations to generalize across regions with similar hydrological characteristics, allowing information from observed locations to influence decision-making in unobserved areas without explicit message passing.

\section{Additional Interpretability Analysis of the Proposed Framework}\label{app:explain}

Here, we visualized two distinct correctly predicted samples, each highlighting meaningful PFAS-related factors, though their top-ranked concepts vary. We present this visualization in figure~\ref{fig:two_figs_1}. This demonstrates that the model dynamically adjusts relevance weighting based on local geographic context, a fundamental principle of the OWL-GPS framework, rather than relying on a uniform global importance pattern. Collectively, these visualizations provide clear evidence that relevance weighting adapts contextually across different geospatial environments.

\begin{figure}[htbp]
    \centering
    \begin{subfigure}[b]{0.48\textwidth}
        \includegraphics[width=\textwidth]{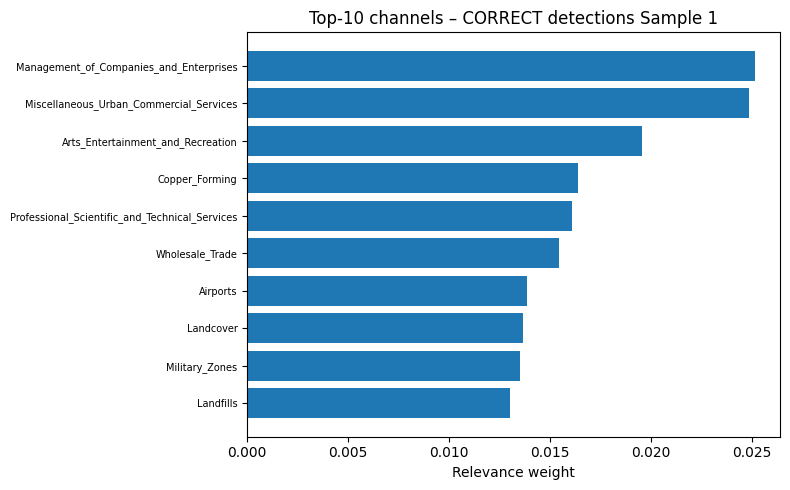}
        \caption{Top-10 Contributing concepts in Region 1.}
        \label{fig:sub1}
    \end{subfigure}\hfill
    \begin{subfigure}[b]{0.48\textwidth}
        \includegraphics[width=\textwidth]{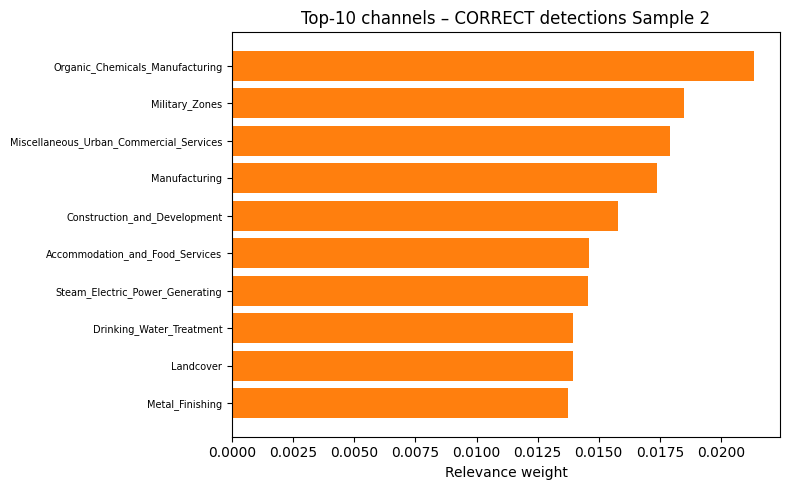}
        \caption{Top-10 Contributing concepts in Region 2.}
        \label{fig:sub2}
    \end{subfigure}
    \caption{Region-aware Decision Making and Analyzing the Mis-detection}
    \label{fig:two_figs_1}
\end{figure}



\section{Additional Details about the Baselines}\label{app:base}
\subsection{Details about the Baselines:}

OWL-GPS introduces a new constrained, streaming, non-replayable problem formulation. Therefore, each baseline must be carefully adapted to comply with the OWL-GPS constraints, such as streaming inputs, no revisitation, limited memory and query budget, and updates only through queried labels. Below, we provide explicit descriptions of how every baseline is trained and evaluated under this unified setup.
\begin{enumerate}
    \item \textbf{Active Learning (AL)}: Successful performance in OWL-GPS requires a delicate balance between exploration and exploitation during inference. AL methods primarily focus on exploration by querying areas of highest model uncertainty. While this ensures thorough coverage, it overlooks exploiting accumulated knowledge about promising regions. Consequently, AL struggles to efficiently navigate the trade-off necessary for high performance in OWL-GPS. For training with the active learning baseline, we adhere strictly to the standard maximum uncertainty strategy for selecting samples to query during a standard supervised training with a ViT. For evaluation, we apply our sampling strategy with one key modification: the selection of queried samples is driven exclusively by the exploration score, with no consideration given to the exploitation score in the sampling process.
    \item \textbf{Greedy Approach (GA)}: GA represents the contrasting extreme by prioritizing exploitation — querying locations where the model predicts the highest likelihood of the target. Purely exploitation-driven strategies like GA risk premature convergence to suboptimal regions and fail to sufficiently explore the environment. Our experimental findings corroborate this limitation, showing inferior performance of GA relative to our method. GA conducts a single forward pass over all incoming regions during evaluation after a standard supervised training phase with a ViT, assigning confidence scores to each. It then selects the top-K regions based on the query budget. GA does not perform online updates, store past regions, or revisit inputs, making it a strict single-pass, non-adaptive baseline within OWL-GPS.
    \item \textbf{Prithvi}: Although Prithvi is a powerful and general framework, it is not specifically designed to handle the multi-faceted challenges of the OWL-GPS problem, such as continual adaptation in dynamic environments. These complex constraints underscore the need for a tailored approach like ours. Specifically, the geospatial foundation model Prithvi is trained using a standard supervised approach. During evaluation, it assigns confidence scores to unobserved samples, selects the highest scoring sample, updates its parameters with the latest observations, and repeats this process for subsequent batches until the query budget is exhausted. 
    \item \textbf{UCB (Bandit exploration)}: At first glance, the OWL-GPS problem may appear similar to a Multi-Armed Bandit (MAB) scenario, making UCB a natural baseline. However, unlike typical MAB settings where each arm is independent, geospatial environments exhibit strong spatial correlations; observing one location informs us about neighboring regions. UCB does not leverage this spatial structure, treating each arm independently. This lack of spatial modeling fundamentally limits its performance in our problem setting, as reflected in Table 2. We would like to emphasize that UCB is adapted to the OWL-GPS setting by treating each incoming region as an 
    \enquote{arm} whose score is computed from its predicted reward (model confidence) and an exploration bonus. Once a region is queried based on the score, the model updates online using the new labeled sample, following the same online meta-update procedure as in our framework. UCB does not access any historical unlabeled pool, and all decisions are made per-arrival as in our framework.
    \item \textbf{Online Meta-Learning (OML)}: OML in our setup follows the same online meta-update rule used in our framework. The key difference is that OML does not use our meta-batch formation strategy with GIA for constructing diverse update sets. Instead, OML directly applies the update using the examples present in our buffer at that moment. Aside from the absence of GIA and our meta-batch formation strategy, OML operates identically to our framework under OWL-GPS constraints: streaming inputs, no revisiting beyond the small buffer, and continuous online meta-adaptation from the sparse labeled data encountered so far.
    \item \textbf{Active Meta-Learning (AML)}: AML follows a training approach similar to ours but differs in two key ways: (a) it forms the meta-training batch using standard random sampling from the buffer, and (b) it lacks an online update mechanism, unlike our proposed strategy. Our experiments demonstrate that AML is ineffective in the OWL-GPS setting, highlighting the critical role of an online update mechanism combined with a coverage-based meta-batch formation strategy.
\end{enumerate}


\section{Extended Related Work}\label{app:ex_rel}

\subsection{Environment Monitoring:}
Prior studies such as~\cite{Sarkar_Lanier_Alfeld_Feng_Garnett_Jacobs_Vorobeychik_2023} and~\cite{sarkar2023partially} have investigated active sampling in related environmental monitoring contexts. However, several fundamental constraints, including the need for strictly online operation, adherence to tight budget limits, memory constraints, and the handling of non-stationary and evolving distributions, distinguish our setting and make these previous approaches unsuitable for OWL-GPS. More generally, there is a rich literature on environmental monitoring methods, but the vast majority focus on settings with static or replayable datasets, abundant labeling, and stationary environments, without addressing the unique challenges posed by our real-world, online discovery scenario. Below, we briefly review key environmental monitoring works to further contextualize the distinctiveness of our approach. Environmental monitoring relies heavily on geospatial analysis, an interdisciplinary approach combining geography, computer science, statistics, and engineering to extract valuable insights from spatial data such as satellite imagery, GPS, and historic datasets. This enables critical applications including understanding land use patterns (\cite{park2023development}), infrastructure planning (\cite{amaral2021environmental};~\cite{regona2024artificial}), and assessing environmental impacts (\cite{mila2023estimating}). These geospatial tools support urban planning for sustainable expansion (\cite{gharaibeh2020improving}) and transportation systems (\cite{kamruzzaman2015investigating}). Recent technological advances, including improved satellite resolution, enhanced GPS accuracy, and increased drone deployment, have transformed data collection and analysis capabilities, facilitating real-time monitoring of key phenomena such as deforestation (\cite{monjardin2020geospatial}), glacier retreat (\cite{thapliyal2023glacier}), and urban heat islands. Additionally, geospatial analysis supports disaster risk management, including wildfire (\cite{shafapourtehrany2023comprehensive}), flood (\cite{liao2023fast}), and earthquake prediction (\cite{lam2021topological}), improving evacuation and relief efforts (\cite{manfre2012analysis}).

Machine learning (ML), deep learning (DL), computer vision (CV), and natural language processing (NLP) have been increasingly integrated with geospatial tools to enhance data processing and predictive power (\cite{casali2022machine}). These methods facilitate large-scale pattern recognition and predictive modeling, crucial for environmental risk assessment and urban sustainability (\cite{son2023algorithmic}). CV enables automated interpretation of satellite and street-view imagery for land classification and change detection (\cite{marasinghe2024computer}), whereas NLP analyses research articles and reports for insights on environmental changes and biodiversity trends (\cite{cao2024multi}). Overall, this integration of geospatial sciences with advanced computational methods paves the way for innovative environmental monitoring solutions, addressing challenges from urbanization to climate change with greater precision and scale. However, our current OWL-GPS problem setting is fundamentally distinct from previously studied problems, capturing the complex nuances of real-world environmental monitoring challenges in a unified framework. 

\subsection{Geospatial Foundational Model}
The development of specialized foundation models has been driven by the need for precision and contextual sensitivity within scientific domains. In geospatial analysis, Geospatial Foundation Models (GFMs) have emerged as powerful tools tailored to interpreting complex Earth surface and atmospheric patterns (\cite{mai2023opportunities}). GFMs address critical challenges such as spatial heterogeneity (\cite{sun2022ringmo}), temporal dynamics (\cite{yao2023ringmo}), and the multidimensional nature of geospatial data (\cite{jakubik2023foundation}). Transformer architectures, especially Vision Transformers (ViT; \cite{dosovitskiy2020image}) and hierarchical designs like Swin Transformers (\cite{liu2021swin}), have become foundational components due to their ability to model long-range dependencies and dynamic attention. Recent innovations enhance GFMs capabilities for spatiotemporal data, for instance, incorporating temporal information as channels (\cite{jakubik2023foundation}) or designing multi-branch networks to capture spatial affinity and temporal continuity (\cite{yao2023ringmo}). Adaptations of pretrained models such as SAM (\cite{kirillov2023segment}) exemplify how conventional image analysis methods are refined for specific geospatial tasks like SAR imagery segmentation (\cite{yan2023ringmo}). Furthermore, Masked Autoencoders (MAE; \cite{he2022masked}) are widely employed for self-supervised training, enabling scalable learning from unlabeled geospatial imagery, while supervised fine-tuning is applied for specialized downstream tasks requiring category-specific outputs (\cite{yan2023ringmo}; \cite{yao2023ringmo}). Among Geospatial Foundation Models, IBM’s Prithvi (\cite{jakubik2023foundation}) is distinguished by its innovative approach to GeoAI and geospatial data analysis, making it a prime candidate for detailed evaluation in this work. Prithvi typically supports six spectral bands: Red, Green, Blue, NIR, SWIR 1, and SWIR 2, enabling richer spectral information capture beyond conventional RGB imagery. This multi-band capability enhances model versatility and effectiveness across diverse geospatial applications, including land cover mapping and environmental monitoring. 


\section{Time and Memory Complexity Details of the Proposed Framework:}\label{app:time}

\subsection{Time Complexity Details:}

The time complexity of our proposed approach primarily depends on the meta-training set formation and policy update steps performed at each query:

\begin{itemize}
    \item At each query step \(t\), the meta-training set is dynamically curated from a core buffer and a reservoir buffer, involving clustering and selection based on sample relevance and diversity. Assuming a buffer size of \(K\) and relevance embedding dimension \(D\), the clustering complexity is approximately \(O(K^{2} \cdot D)\).
    
    \item The meta-policy update consists of gradient-based optimization over the meta-training batch. The complexity scales linearly with the batch size \(B\), the parameter count of the concept encoder, relevance encoder, and decoder, as well as the cost of forward and backward passes. Denoting the total cost of forward and backward passes as \(P\), this step has complexity \(O(B \cdot P)\).
    
    \item Each query also requires computing exploitation and exploration scores over the unlabeled pool of size \(N\) to select the next sample. The complexity for scoring each candidate region involves \(O(N^{2} \cdot D + P)\), leading to a total complexity of \(O(N^{2} \cdot D + N \cdot P)\).
\end{itemize}
In summary, the overall time complexity per query step can be expressed as:
\[
O(K^{2} \cdot D) + O(B \cdot P) + O(N^{2} \cdot D + N \cdot P) \approx O(N^{2} \cdot D) \quad \text{(assuming } K \ll N\text{)} \approx O(N^{2}) 
\]
(assuming  $D \ll N\text{)}$. 

Furthermore, we present the time and memory costs associated with our proposed framework as follows:  
\begin{itemize}
    \item On an NVIDIA A100 GPU, sampling step during inference takes on average approximately 2.1 minutes (128.8 seconds). The model’s decision-making step is lightweight, requiring only ~4 GB of GPU memory.
\end{itemize}
\subsection{Details about Compute Resource and Code}
All experiments were conducted using PyTorch on an NVIDIA A100 GPU. We also provide the code and data for reference in the following anonymous GitHub link~\footnote{\href{https://anonymous.4open.science/r/OWL-GPS_NeurIPS-65BA/README.md}{Code}}.

\section{Training and Inference Pseudocode}
We present training and inference pseudocode in Algorithm~\ref{alg:training}, and Algorithm~\ref{alg:inference}, respectively.
\begin{algorithm}[H]
\caption{Relevance-Guided Online Meta-Training with Latent  Concepts for Geospatial Discovery (\textcolor{cyan}{\textbf{Training)}}}
\label{alg:training}
\begin{algorithmic}[1]
\Require Training pool $\mathcal{X}_{\text{train}}$, Query budget $\mathcal{C}_{\text{train}}$
\Require Concept encoder $\mathrm{CE}_{\psi}$, relevance encoder $\mathrm{RE}_{\zeta}$, decoder $\mathrm{RD}_{\phi}$
\Require Core buffer $\mathcal{D}^{\text{core}}=\{\}$, reservoir buffer $\mathcal{D}^{\text{reservoir}}=\{\}$
\State Initialize parameters $\theta =\{\zeta, \phi\}$, step $t \gets 1$, and $\psi$ is frozen.
\While{$t \le \mathcal{C}_{\text{train}}$}
    \State \textbf{(Training-time sampling)}
    \For{each unlabeled region $x \in \mathcal{X}_{\text{train}}$}
        \State $\tilde{c}(x) \gets \mathrm{CE}_{\psi}(x)$
        \State $c(x) \gets \mathrm{GS}(\tilde{c_x})$
        \State $(\mu_{\zeta}(c(x)), \sigma_{\zeta}(c(x))) \gets \mathrm{RE}_{\zeta}(c(x))$
        \State Sample $r(c(x)) \sim \mathcal{N}(\mu_{\zeta}(c(x)), \mathrm{diag}(\sigma^2_{\zeta}(c(x))))$
        \State $\pi_{\theta}(x) \gets \mathrm{RD}_{\phi}(c(x), r(c(x)))$
        \State $w_x \gets$ $\sum_{x' \in \{k\}}$ $\lVert \mu_{\zeta}(c(x)) - \mu_{\zeta}(c(x')) \rVert_2^2; k = \{ x^{(q_1)},\ldots, x^{(q_{t-1})}\}$ ; (similarity weight to previously queried regions)
        \State $\exp^{(w_x)} \gets$ \emph{relevance uncertainty} score of region $x$; (see Equation 6)
        \State $u_{\text{dec}}(x) \gets$ \emph{prediction uncertainty} of region $x$ using $\sum_{i=1}^{P} \exp^{\bigl(-\lvert \pi_{\theta}(x_i) - 0.5 \rvert\bigr)}$ (see equation 6)

        \State \textcolor{cyan}{$\text{score}_{\text{train}}(x) \gets \exp^{(w_x)}\, u_{\text{dec}}(x)$}
    \EndFor
    \State \textcolor{cyan}{$x^{(q_t)} \gets \arg\max_{x \in \{ \text{Set of Unobserved regions}\}}  \text{score}_{\text{train}}(x)$}
    \State $y^{(q_t)} \gets \text{Observe ground truth labels corresponds to the region } x^{(q_t)}$.

    \State \textbf{(Buffer update)}
    \State Insert $(x^{(q_t)}, y^{(q_t)})$ into $\mathcal{D}^{\text{core}}$ with $(\mathrm{duration}=0, \mathrm{count}=0)$
    \State Update durations in $\mathcal{D}^{\text{core}}$ and $\mathcal{D}^{\text{reservoir}}$, evict by lifespan
    \State Move evicted core samples to $\mathcal{D}^{\text{reservoir}}$, drop excess reservoir samples

    \State \textbf{(Meta-batch formation)}
    \State \emph{Cluster} $\mathcal{D}_t^{\text{core}}$ elements in relevance space
    \State From each cluster, select a sample maximizing $\exp^{(\mathrm{duration}/(\mathrm{count}+1))}$ to form $D_t^{\mathrm{core}}$
    \State Sample $K$ elements from $\mathcal{D}_t^{\mathrm{reservoir}}$ with weight $\propto \exp^{(\text{duration}/(\text{count}+1))}$ to form $D_t^{\mathrm{reservoir}}$
    \State $\mathcal{D}_t \gets \mathcal{D}_t^{\mathrm{core}} \cup \mathcal{D}_t^{\mathrm{reservoir}}$, randomly split $\mathcal{D}_t = \mathcal{D}_t^{\mathrm{train}} \cup \mathcal{D}_t^{\mathrm{test}}$

    \State \textbf{(Meta-update)}
    \State $\theta' \gets \theta$
    \State $\mathcal{L}_{\text{inner}} \gets \sum_{(x,y)\in \mathcal{D}_t^{\mathrm{train}}} \big[\mathrm{Focal}(\pi_{\theta}(x), y) + \beta \,\mathrm{KL}(\mathcal{N}(\mu_{\zeta}(c(x)), \mathrm{diag}(\sigma^2_{\zeta}(c(x))))\,\|\,\mathcal{N}(0, \mathrm{diag}(I)))\big]$
    \State $\theta' \gets \theta' - \eta \nabla_{\theta'} \mathcal{L}_{\text{inner}}$
    \State $\mathcal{L}_{\text{outer}} \gets \sum_{(x,y)\in D_t^{\text{test}}} \big[\mathrm{Focal}(\pi_{\theta'}(x), y) + \beta \,\mathrm{KL}(\mathcal{N}(\mu_{\zeta'}(c(x)), \mathrm{diag}(\sigma^2_{\zeta'}(c(x))))\,\|\,\mathcal{N}(0, \mathrm{diag}(I)))\big]$
    \State $\theta \gets \theta - \eta \nabla_{\theta} \mathcal{L}_{\text{outer}}$

    \State \textbf{(Online Single-step update on latest query)}
    \State Compute $\mathcal{L}_{\text{latest}}$ on $(q_t, y_{q_t})$ and update $\theta \gets \theta - \eta \nabla_{\theta} \mathcal{L}_{\text{latest}}$

    \State Increment counts for all samples used in $\mathcal{D}_t$
    \State $t \gets t + 1$
\EndWhile
\State \textbf{Return} $\theta =\{\zeta,\phi\}.$
\end{algorithmic}
\end{algorithm}

\begin{algorithm}[H]
\caption{Relevance-Guided Online Meta-Training with Latent  Concepts for Geospatial Discovery (\textcolor{orange}{\textbf{Inference})}}
\label{alg:inference}
\begin{algorithmic}[1]
\Require Test pool $\mathcal{X}_{\text{test}}$, Query budget $\mathcal{C}_{\text{test}}$
\Require Concept encoder $\mathrm{CE}_{\psi}$, relevance encoder $\mathrm{RE}_{\zeta}$, decoder $\mathrm{RD}_{\phi}$
\Require Core buffer $\mathcal{D}^{\text{core}}=\{\}$, reservoir buffer $\mathcal{D}^{\text{reservoir}}=\{\}$
\State Initialize parameters $\theta =\{\zeta, \phi\}$, Initialize unqueried set $\mathcal{U} \gets \mathcal{X}_{\text{test}}$, step $t \gets 1$, and $\psi$ is frozen.
\While{$t \le \mathcal{C}_{\text{test}}$ \textbf{and} $\mathcal{U} \neq \emptyset$}
    \State \textbf{(Scoring)}
    \For{each $x \in \mathcal{U}$}
        \State $\tilde{c}(x) \gets \mathrm{CE}_{\psi}(x)$
        \State $c(x) \gets \mathrm{GS}(\tilde{c_x})$
        \State $(\mu_{\zeta}(c(x)), \sigma_{\zeta}(c(x))) \gets \mathrm{RE}_{\zeta}(c(x))$
        \State Sample $r(c(x)) \sim \mathcal{N}(\mu_{\zeta}(c(x)), \mathrm{diag}(\sigma^2_{\zeta}(c(x))))$
        \State $\pi_{\theta}(x) \gets \mathrm{RD}_{\phi}(c(x), r(c(x)))$
        \State $w_x \gets$ $\sum_{x' \in \{k\}}$ $\lVert \mu_{\zeta}(c(x)) - \mu_{\zeta}(c(x')) \rVert_2^2; k = \{ x^{(q_1)},\ldots, x^{(q_{t-1})}\}$ ; (similarity weight to previously queried regions)
        \State $\exp^{(w_x)} \gets$ \emph{relevance uncertainty} score of region $x$; (see Equation 6)
        \State $u_{\text{dec}}(x) \gets$ \emph{prediction uncertainty} of region x using $\sum_{i=1}^{P} \exp^{\bigl(-\lvert \pi_{\theta}(x_i) - 0.5 \rvert\bigr)}$ (see equation 6)
        \State \textcolor{orange}{$\text{Exploit}_{ \pi_{\theta}}^{\mathrm{score}}(x) \gets \exp^{(-w_x)}\,\sum_{i=1}^{P} \exp^{\bigl( \pi_{\theta}(x_i) \bigr)}$} (see Equation 8)
        \State \textcolor{orange}{$\text{Explore}_{ \pi_{\theta}}^{\mathrm{score}}(x) \gets \exp^{(w_x)}\, u_{\text{dec}}(x)$} (see Equation 8)
        \State \textcolor{orange}{$\mathrm{Score}_{ \pi_{\theta}}(x) \gets \kappa(\mathcal{C}) \cdot \mathrm{Explore}_{ \pi_{\theta}}^{\mathrm{score}}(x) + (1 - \kappa(\mathcal{C})) \cdot \mathrm{Exploit}_{ \pi_{\theta}}^{\mathrm{score}}(x)$}  (see Equation 9)
    \EndFor
    \State \textcolor{orange}{$x^{(q_t)} \gets \arg\max_{x \in \{ \text{Set of Unobserved regions}\}}  \mathrm{Score}_{ \pi_{\theta}}(x)$}
    \State $y^{(q_t)} \gets \text{Observe ground truth labels corresponds to the region } x^{(q_t)}$.
    \State \textbf{(Compute Evaluation Metrics)}
    \State \textcolor{orange}{\textsc{RecordTargets}$(x^{(q_t)}, y^{(q_t)})$ \text{update SR and other metrics, such as F-Score, Accuracy, Precision, Recall.}}
    \State \textbf{(Buffer update)}
    \State Insert $(x^{(q_t)}, y^{(q_t)})$ into $\mathcal{D}^{\text{core}}$ with $(\mathrm{duration}=0, \mathrm{count}=0)$
    \State Update durations in $\mathcal{D}^{\text{core}}$ and $\mathcal{D}^{\text{reservoir}}$, evict by lifespan
    \State Move evicted core samples to $\mathcal{D}^{\text{reservoir}}$, drop excess reservoir samples

    \State \textbf{(Meta-batch formation)}
    \State \emph{Cluster} $\mathcal{D}_t^{\text{core}}$ elements in relevance space
    \State From each cluster, select a sample maximizing $\exp^{(\mathrm{duration}/(\mathrm{count}+1))}$ to form $D_t^{\mathrm{core}}$
    \State Sample $K$ elements from $\mathcal{D}_t^{\mathrm{reservoir}}$ with weight $\propto \exp^{(\text{duration}/(\text{count}+1))}$ to form $D_t^{\mathrm{reservoir}}$
    \State $\mathcal{D}_t \gets \mathcal{D}_t^{\mathrm{core}} \cup \mathcal{D}_t^{\mathrm{reservoir}}$, randomly split $\mathcal{D}_t = \mathcal{D}_t^{\mathrm{train}} \cup \mathcal{D}_t^{\mathrm{test}}$

    \State \textbf{(Meta-update)}
    \State $\theta' \gets \theta$
    \State $\mathcal{L}_{\text{inner}} \gets \sum_{(x,y)\in \mathcal{D}_t^{\mathrm{train}}} \big[\mathrm{Focal}(\pi_{\theta}(x), y) + \beta \,\mathrm{KL}(\mathcal{N}(\mu_{\zeta}(c(x)), \mathrm{diag}(\sigma^2_{\zeta}(c(x))))\,\|\,\mathcal{N}(0, \mathrm{diag}(I)))\big]$
    \State $\theta' \gets \theta' - \eta \nabla_{\theta'} \mathcal{L}_{\text{inner}}$
    \State $\mathcal{L}_{\text{outer}} \gets \sum_{(x,y)\in D_t^{\text{test}}} \big[\mathrm{Focal}(\pi_{\theta'}(x), y) + \beta \,\mathrm{KL}(\mathcal{N}(\mu_{\zeta'}(c(x)), \mathrm{diag}(\sigma^2_{\zeta'}(c(x))))\,\|\,\mathcal{N}(0, \mathrm{diag}(I)))\big]$
    \State $\theta \gets \theta - \eta \nabla_{\theta} \mathcal{L}_{\text{outer}}$

    \State \textbf{(Online Single-step update on latest query)}
    \State Compute $\mathcal{L}_{\text{latest}}$ on $(q_t, y_{q_t})$ and update $\theta \gets \theta - \eta \nabla_{\theta} \mathcal{L}_{\text{latest}}$

    \State Increment counts for all samples used in $\mathcal{D}_t$
    \State $\mathcal{U} \gets \mathcal{U} \setminus \{x^{(q_t)}\}$
    \State $t \gets t + 1$
\EndWhile
\State \textbf{Return} \text{Discovered targets and final metrics.}
\end{algorithmic}
\end{algorithm}

\section{Experimental Setup of Domain-Informed Greedy Baseline}

Instead of running a simulation like SWAT directly, which is designed for general watershed hydrology modeling but lacks established parameterization for PFAS transport, we implemented a streamlined Python-based workflow more tailored to PFAS-related surface water contamination in collaboration with Water Quality models' experts. This approach allowed us to integrate domain-relevant features without relying on assumptions unsupported by current PFAS science.

\begin{figure}[ht!]
  \centering
  \includegraphics[width=0.5\columnwidth]{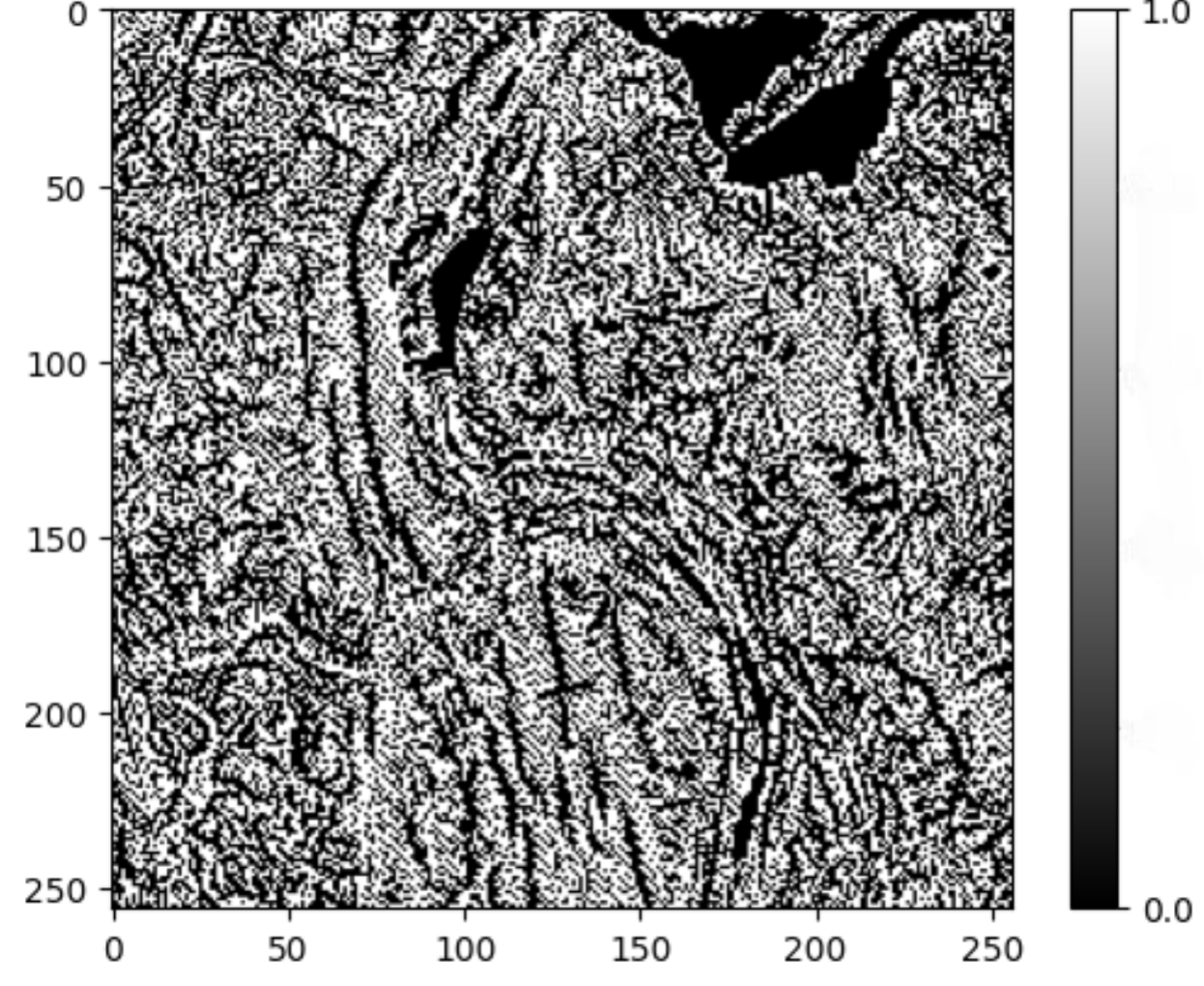}
  \caption{Example binary output from the pollutant transport simulation: 0 represents low contamination, and 1 represents high contamination. The simulation approximates the distribution of PFAS contamination based on hydrological and environmental parameters, providing a practical alternative to full-scale SWAT simulations.}
  \label{fig:simulation_output_example}
\end{figure}

First, we assign initial pollutant concentrations in each patch by setting cells flagged as “dischargers” to a high baseline value (e.g., 100), while land cover–based default concentrations provide moderately-low to low initial contamination levels for other cells. We then construct a Hydrologic Response Unit (HRU) parameter table that assigns infiltration and runoff values based on a cell’s land cover, soil type, and slope. Using these parameters, each patch’s land cover and soil rasters define infiltration/runoff ratios for every grid cell. We incorporate two key rasters; flow direction and flow accumulation, both exported from System for Automated Geoscientific Analyses (SAGA); a GIS and geospatial analysis tool focused on geospatial data processing, analysis, and visualization, to specify how water and pollutant mass move downstream. In each iteration, a fraction of the pollutant mass in each cell is transferred to its downstream neighbor according to the cell’s infiltration/runoff factors, guided by the flow direction raster and scaled by the flow accumulation values. Repeating this process causes pollutant mass to concentrate in lower-lying or higher-flow cells, thereby modeling how contamination evolves over time within the patch.

After the simulation converges, we produce a final pollutant concentration raster, thresholded by the median value of concentrations across all patches for the year to generate a binary contamination map, as illustrated in Fig.~\ref{fig:simulation_output_example}. The binary outputs, with values of 0 and 1 representing low and high contamination respectively, visualizes the simulation's predictions of contamination distribution. These outputs are compared against test set patches containing actual observed PFAS presence \textit{in surface water} to evaluate how accurately the simulation captures observed contamination patterns. Although this standalone approach remains computationally non-trivial, it is more tractable for batch processing of multiple patches than running a full SWAT project repeatedly. Consequently, it provides a practical baseline for pollutant transport modeling within our broader framework, allowing us to compare simulated outputs to both real-world data and other modeling approaches.

\section{Impact Statement}
This work advances methods for learning and decision-making under extreme data scarcity and resource constraints, motivated by real-world geospatial discovery tasks such as environmental monitoring and public health risk identification. By enabling models to adapt online using limited, costly observations, our approach has the potential to support more efficient allocation of field sampling resources, reduce monitoring costs, and accelerate the identification of regions requiring further investigation.

In applied settings such as pollution or environmental risk mapping, improved sample efficiency may help inform earlier intervention, guide regulatory attention, and support community-level decision-making when comprehensive data collection is infeasible. At the same time, the framework is designed to complement, not replace, expert judgment and established scientific workflows, and its outputs should be interpreted as decision-support signals rather than definitive assessments.

Potential risks include misuse or over-interpretation of model predictions in high-stakes contexts, particularly if deployed without appropriate domain oversight or validation. To mitigate this, our framework emphasizes interpretability through concept-guided relevance and is evaluated using publicly available data and transparent metrics. We encourage responsible use of such models alongside domain expertise and community engagement, especially in settings with regulatory or societal implications. Importantly, our approach does not rely on large-scale pretraining or extensive retraining, and, in addition to addressing data and resource constraints, may help constrain computational resource requirements in geospatial monitoring pipelines.

Overall, we believe this work contributes positively by providing tools for more responsible and efficient learning in data-limited environments, while acknowledging the importance of careful deployment and human-in-the-loop decision-making.

\newpage

\end{document}